\documentclass{tlp}

\usepackage{pdfsync} \synctex=1
\usepackage{todonotes} 

\usepackage{verbatim} 
\usepackage{enumitem}
\usepackage{amsmath}
\usepackage{graphicx}
\usepackage{latexsym}
\usepackage{xspace}
\usepackage{url}

\newtheorem{lemma}{Lemma}

\newtheorem{proposition}{Proposition}
\newtheorem{corollary}{Corollary}
\newtheorem{theorem}{Theorem}

\newtheorem{definition}{Definition}

\def \yu#1{\todo[inline,color=blue!40]{#1}}

   \newcommand{\alltypes}{\ensuremath{T}}
   \newcommand{\classic}{\ensuremath{cla}}
   \newcommand{\supported}{\ensuremath{sup}}
   \newcommand{\stable}{\ensuremath{sta}}
   \newcommand{\atypeofmodel}{\ensuremath{w}}
   \newcommand{\false}{\mathit{false}}
   \newcommand{\true}{\mathit{true}}

\newcommand{\decisionsuperscript}{\Delta}
\newcommand{\decision}[1]{#1^{\decisionsuperscript}}
\newcommand{\opposite}[1]{\overline{#1}}
\newcommand{\terminalstate}[1]{Ok(#1)}
\newcommand{\failstate}{Failstate}
\newcommand{\transitionarrow}{\Longrightarrow}

\newcommand{\aprogram}{\mathit{\Pi}}

\newcommand{\anatom}{a}
\newcommand{\anotheratom}{b}
\newcommand{\athirdatom}{c}
\newcommand{\abody}{B}
\newcommand{\ahead}{A}
\newcommand{\aliteral}{l}
\newcommand{\anotherliteral}{r}
\newcommand{\astringofliterals}{L}
\newcommand{\anotherstringofliterals}{R}
\newcommand{\asetofatoms}{X}
\newcommand{\anothersetofatoms}{Y}
\newcommand{\asetofliterals}{M}
\newcommand{\amodel}{\asetofliterals}
\newcommand{\asetofmodels}{N}
\newcommand{\aclause}{C}
\newcommand{\aconjunction}{D}
\newcommand{\aformula}{F}

 \def\beq{\begin{equation}}
 \def\eeq#1{\label{#1}\end{equation}}
 \def\ba{\begin{array}}
 \def\ea{\end{array}}
\newcommand{\length}[1]{|#1|}
\newcommand{\logicaltrue}{\top}
\newcommand{\logicalfalse}{\bot}
\newcommand{\logicalor}{\vee}
\newcommand{\logicaland}{\wedge}
\newcommand{\logicalnot}{\neg}

\newcommand{\setintersection}{\cap}
\newcommand{\setunion}{\cup}
\newcommand{\aspnot}{not~}
\newcommand{\aspimplication}{\leftarrow}
\newcommand{\anindex}{i}
\newcommand{\anotherindex}{j}
\newcommand{\athirdindex}{k}
\newcommand{\afourthindex}{n}
\newcommand{\lexlower}{\leq_{lex}}
\newcommand{\npcomplexity}{\sc np}
\newcommand{\restriction}[2]{{#1}_{|{#2}}}
\newcommand{\disjunctionof}[1]{{#1}^{\logicalor}}

\newcommand{\bodies}[1]{Bodies(#1)}

\newcommand{\atoms}[1]{{atoms(#1)}}
\newcommand{\completion}[1]{\thecompletion(#1)}
\newcommand{\thecompletion}{comp}
\newcommand{\reduct}[2]{#1^{#2}}

\newcommand{\seenasamodel}[1]{#1}
\newcommand{\answersetof}[1]{#1^{+}}

\newcommand{\dependingprogram}[2]{\thedependingprogram(#1,#2)}
\newcommand{\thedependingprogram}{t}
\newcommand{\mainprogram}[1]{\themainprogram(#1)}
\newcommand{\themainprogram}{g}

\newcommand{\depth}[1]{v(#1)}

\newcommand{\auxvariable}[1]{\alpha_{#1}}

\newcommand{\agraph}{G}
\newcommand{\ordinarydp}[1]{DP_{#1}}
\newcommand{\ordinarydpt}[2]{DPT_{#1,#2}}

\newcommand{\firstgraph}[2]{DP^2_{#1,#2}}
\newcommand{\smgraph}[1]{SM^2_{#1}}

\newcommand{\dpgraph}[1]{DP^2_{#1}}

\newcommand{\unitdpll}{\hbox{$\mathit{Unit}$}}
\newcommand{\unitdpllshort}{\hbox{$\mathit{Unit}$}}
\newcommand{\decidedpll}{Decide}
\newcommand{\faildpll}{Conclude}
\newcommand{\backtrackdpll}{Backtrack}
\newcommand{\okruledpll}{Success}
\newcommand{\unfoundeddpll}{\mathit{Unfounded}}
\newcommand{\unitgen}{\hbox{$\mathit{Unit}$}}
\newcommand{\unitleft}{\unitgen_\leftstate}
\newcommand{\decideleft}{Decide_\leftstate}
\newcommand{\failleft}{Conclude_\leftstate}
\newcommand{\backtrackleft}{Backtrack_\leftstate}
\newcommand{\unitright}{\unitgen_\rightstate}
\newcommand{\decideright}{Decide_\rightstate}
\newcommand{\failright}{Conclude_\rightstate}
\newcommand{\backtrackright}{Backtrack_\rightstate}
\newcommand{\failcross}{Conclude_{\rightstate\leftstate}}
\newcommand{\backtrackcross}{Backtrack_{\rightstate\leftstate}}

\newcommand{\propagateleft}{\propagategen_\leftstate}
\newcommand{\propagateright}{\propagategen_\rightstate}
\newcommand{\propagategen}{Propagate}

\newcommand{\crossrule}{Cross_{\leftstate\rightstate}}

\newcommand{\propunitpropagate}[3]{\uniformpcondition{\theunit}{#3}{#1}{#2}}
\newcommand{\propallcancel}[3]{\uniformpcondition{\theallcancel}{#3}{#1}{#2}}
\newcommand{\propbacktrue}[3]{\uniformpcondition{\thebacktrue}{#3}{#1}{#2}}
\newcommand{\propunfounded}[3]{\uniformpcondition{\theunfounded}{#3}{#1}{#2}}

\newcommand{\theunit}{\hbox{$\mathit{UnitPropagate}$}}
\newcommand{\theallcancel}{AllRulesCancelled}
\newcommand{\thebacktrue}{BackchainTrue}
\newcommand{\theunfounded}{\mathit{Unfounded}}

\newcommand{\simpletemplate}[4]{STT^{#1,#3}_{#2,#4}}

\newcommand{\uniformpcondition}[4]{{#4}\in\setofuniformpcondition{#1}{#2}{#3}}
\newcommand{\setofuniformpcondition}[3]{#1(#2,#3)}

\newcommand{\acondition}[3]{\setofuniformpcondition{\thecondition}{#3}{#1}{#2}}

\newcommand{\outprop}[3]{{#1}(#2,#3)}
\newcommand{\thecondition}{p}
\newcommand{\asetofpropagators}{\mathcal{P}}
\newcommand{\leftstate}{\mathcal{L}}
\newcommand{\rightstate}{\mathcal{R}}
\newcommand{\aside}{s}

\newcommand{\smpropagators}{sm}
\newcommand{\smdisjpropagators}{sd}

\newcommand{\uppropagators}{up}

\newcommand{\markercmodels}{C}
\newcommand{\markergnt}{G}
\newcommand{\markerdlv}{D}

\newcommand{\genrighttest}{t}

\newcommand{\thetestcmodels}{\genrighttest^{\markercmodels}}
\newcommand{\testcmodels}[2]{\thetestcmodels(#1,#2)}

\newcommand{\thetestgnt}{\genrighttest^{\markergnt}}
\newcommand{\testgnt}[2]{\thetestgnt(#1,#2)}

\newcommand{\thetestdlv}{\genrighttest^{\markerdlv}}
\newcommand{\testdlv}[2]{\thetestdlv(#1,#2)}

\newcommand{\genleftgen}{g}
\newcommand{\gencmodels}[1]{\thegencmodels(#1)}
\newcommand{\thegencmodels}{\genleftgen^{\markercmodels}}
\newcommand{\gengnt}[1]{\thegengnt(#1)}
\newcommand{\thegengnt}{\genleftgen^{\markergnt}}
\newcommand{\thegendlv}{\genleftgen^{\markerdlv}}
\newcommand{\gendlv}[1]{\thegendlv(#1)}

\newcommand{\identity}{\thegendlv}

\newcommand{\thecnfcomp}{\mathit{cnfcomp}}

\newcommand{\smodels}{{\sc smodels}\xspace}
\newcommand{\cmodels}{{\sc cmodels}\xspace}
\newcommand{\clasp}{{\sc clasp}\xspace}
\newcommand{\wasp}{{\sc wasp}\xspace}
\newcommand{\dlv}{{\sc dlv}\xspace}
\newcommand{\gnt}{{\sc gnt}\xspace}

\newcommand{\dpll}{{\sc dpll}\xspace}

\begin{document}

 \submitted{9 April 2015}
 \revised{}
 \accepted{2 August 2015}

\title{Disjunctive
Answer Set Solvers via Templates
}
\author[R. Brochenin, Y. Lierler and M. Maratea]{
REMI BROCHENIN\\
University of Genova, Italy\\
\email{remi.brochenin@unige.it}\\
\and YULIYA LIERLER\\
University of Nebraska at Omaha\\
\email{ylierler@unomaha.edu}\\
\and MARCO MARATEA\\
University of Genova, Italy\\
\email{marco@dibris.unige.it} }
\maketitle

\begin{abstract}
Answer set programming  is a declarative programming paradigm oriented
towards difficult combinatorial search problems. A fundamental task in answer
set programming is to compute stable models, i.e., solutions of logic programs.
Answer set solvers are the programs that perform this task. The problem of
deciding whether a disjunctive program has a stable model is $\Sigma^P_2$-complete. 
The high complexity of reasoning within disjunctive logic programming   
is responsible for few solvers capable of dealing with such programs,
namely {\dlv}, {\gnt}, {\cmodels}, {\clasp} and {\wasp}. 
 In this paper we show that transition systems introduced by
Nieuwenhuis, Oliveras, and Tinelli to model and analyze satisfiability
solvers
can be adapted for disjunctive answer set solvers. Transition systems give a unifying  perspective and bring clarity
in the description and comparison of solvers. They can be effectively used for
analyzing, comparing and proving  correctness of search algorithms as
well as inspiring new ideas in the design of disjunctive answer set solvers. In this light, we introduce a general
template, which accounts for major techniques 
implemented in disjunctive
solvers. We then illustrate how this general template
captures solvers {\dlv}, {\gnt}, and {\cmodels}. 
We also show how this framework provides a convenient tool for designing new
solving algorithms by means of combinations of techniques employed in different
solvers. To appear in Theory and Practice of Logic
Programming (TPLP). 
\end{abstract}

\begin{keywords}
Answer Set Programming, Abstract Solvers
\end{keywords}

\section{Introduction}
Answer set programming~\cite{mar99,nie99,bar03,eite-etal-97f,gel88,gel91b} is
a declarative programming paradigm oriented towards difficult combinatorial search
  problems. The idea of answer set programming (ASP)
  is to represent a given problem with a logic program, whose answer sets
  correspond to solutions of the problem (see e.g.,~\citeNP{lif99a}). ASP has
  been applied to solve problems in various areas of science and technology including 
  graph-theoretic problems arising in zoology and
  linguistics~\cite{bro07}, team building problems in container
  terminal~\cite{ricca-etal-tplp-2012}, and product
  configuration tasks~\cite{soin-niem-99}.
  A fundamental task in ASP is to compute stable models of logic programs.
  Answer set solvers are the programs that perform this task.
  There were sixteen answer set solvers participating in the recent 
  Fifth Answer Set Programming
  Competition\footnote{
  \url{https://www.mat.unical.it/aspcomp2014/FrontPage#Participant_Teams}}.

Gelfond and Lifschitz  introduced
  logic programs with disjunctive rules~\cite{gel91b}.
  The problem of deciding whether a disjunctive program has a stable model is
  $\Sigma^P_2$-complete~\cite{eit93a}. 
  The problem of deciding  whether a non-disjunctive program has a stable model is
  NP-complete.
  The high complexity of reasoning within disjunctive logic programming  
  stems from two sources: first, there is a potentially exponential number of
  candidate models, and, second, the hardness of checking whether a
  candidate model is a stable model
  of a propositional disjunctive logic program is co-NP-complete. 
  Only five answer set systems 
  can solve disjunctive programs: {\dlv}~\cite{dlv03a}, {\gnt}~\cite{jan06}, {\cmodels}~\cite{lie05},
  {\clasp}~\cite{geb13a} and {\wasp}~\cite{dod13}.

Several formal approaches have been used to
  describe and compare search procedures implemented in answer set
  solvers. These approaches range from a pseudo-code
  representation of the procedures~\cite{giumar05,giu08}, to tableau
  calculi~\cite{gebsch06iclp,geb13}, to abstract frameworks via transition
  systems~\cite{lier08,lier10,lie11a}.
  The latter method originates from the work by
  \citeN{nie06}, where authors  propose to use transition systems to describe the 
  Davis-Putnam-Logemann-Loveland ({\dpll}) procedure~\cite{dav62}.   
  Nieuwenhuis et al. introduce an abstract framework called {\dpll} graph, that captures
  what states of computation are, and what transitions 
  between states are allowed.
  Every execution of the  {\dpll} procedure 
  corresponds to a path in the {\dpll} graph.
  Some edges may correspond to unit
  propagation steps, some to branching, some to backtracking.  

Such an abstract way of presenting  algorithms
  simplifies their analysis.
  This approach  has been adapted \cite{lier10,lie11a} to describing
  answer set solvers for {\em non-disjunctive} programs including
  {\smodels}, {\cmodels}, and {\clasp}.
  This type of graphs has been used to relate
  algorithms in precise mathematical terms.  Indeed, once we  represent
  algorithms via graphs, comparing the graphs translates into
  studying the relationships of underlying algorithms.
  More generally,
  the unifying perspective of transition systems
  brings clarity in the description and comparison of solvers.
  Practically, such graph representations may serve as 
  an effective tool for
  analyzing, comparing, proving correctness of, and reasoning formally about
  the underlying search algorithms. It may also inspire new ideas
  in the design of solvers.
  
In this paper we present transition
  systems that suit multiple {\em disjunctive} answer  set solvers. 
  We define a general framework,  a {\em graph template}, which accounts for
  major techniques implemented in disjunctive
  answer set solvers excluding backjumping and learning.
  We study  formal
  properties of this template and we use the template to describe {\gnt}, {\cmodels} and {\dlv} implementing plain
  backtracking.
We then show how a graph template facilitates a design of new solving algorithms by means of combinations of techniques employed in different solvers.
For instance, we present a new abstract solver 
that can be seen as a hybrid between {\cmodels} and {\gnt}.
  We also present how different solvers may be compared by means of transition systems. In particular, we illustrate a close 
  relationship between answer set solvers {\dlv} and {\cmodels} through
  the related graphs.
  The fact that proposed framework does not account for backjumping and learning is one of the reasons that prevents us from capturing such advanced disjunctive answer set solvers as {\clasp} and {\wasp}.
  It is a direction of future work to investigate how the proposed framework can be adjusted to accommodate these solvers in full generality.
The paper is structured as follows. Section~\ref{sec:prel} introduces
  required preliminaries. Section~\ref{sec:cm} presents a first
  abstract solver related to {\cmodels}.
  Section~\ref{sec:NewPresentation} defines our general template that  accounts
  for techniques implemented in disjunctive solvers, and
  Section~\ref{sec:applications} uses this template to define abstract frameworks for
  disjunctive solvers.
 Proofs are presented in Section~\ref{sec:proofs}.
  Section~\ref{sec:conclusion} discusses related work and concludes with the final remarks.


The current paper builds on the content presented by \citeN{blm14}. 
It enhances the earlier work by introducing notions of a graph template,
``propagator conditions'', and ``approximating pairs`` that allow to more uniformly account 
for major  techniques  implemented in disjunctive answer
        set solvers. Complete proofs of the formal results are also provided.

\section{Preliminaries}
\label{sec:prel}
\subsection{Formulas, Logic Programs, and Program's Completion}
\paragraph{Formulas.}
{\em Atoms} are Boolean
variables over $\{\true,\false\}$. The symbols $\logicalfalse$ and
$\logicaltrue$ are the $\false$ and the $\true$ constants, respectively.
The letter $\aliteral$ denotes a literal, that is an atom $\anatom$ or its
negation $\logicalnot \anatom$, and
$\opposite{\aliteral}$ is the complement of $\aliteral$, i.e.,
literal $\anatom$ for $\logicalnot \anatom$ and literal $\logicalnot \anatom$ for $\anatom$.
 Propositional formulas are logical expressions defined over
atoms and symbols $\logicalfalse$, $\logicaltrue$ in usual way.
A finite disjunction of literals
is a {\em clause}.
We identify an empty clause with the symbol $\logicalfalse$.
A conjunction (resp. a disjunction) of literals will sometimes be seen as a set,
containing each of its literals.
Since a clause is identified with a set of its literals, there are no
 repetition of literals in a clause.
A {\em CNF formula} is a finite conjunction (alternatively, a set) of
clauses.
Since a CNF formula is identified with a set of clauses, there are no
 repetition of clauses in a CNF formula.

For a conjunction (resp. a disjunction)
$\aconjunction$ of literals, by
$\opposite{\aconjunction}$ we denote the disjunction (resp. the conjunction) of
the complements of the elements of $\aconjunction$.
For example, $\opposite{\anatom\logicalor\logicalnot
 \anotheratom}$ denotes $\logicalnot \anatom \logicaland
\anotheratom$, while
$\opposite{\anatom\logicaland\logicalnot
 \anotheratom}$ denotes $\logicalnot \anatom \logicalor \anotheratom$.
For a set~$\astringofliterals$ of literals, by
$\astringofliterals^{\logicalor}$ we denote the disjunction of its elements and
$\astringofliterals^{\logicaland}$ the conjunction of its elements;
by
 $\atoms{\astringofliterals}$ we denote the set of atoms occurring in $\astringofliterals$.
For a set $\asetofmodels$ of sets of literals by
 $\atoms{\asetofmodels}$ we denote the set of atoms occurring in the elements of $\asetofmodels$.
 For example, $\atoms{\{\anatom,\logicalnot\anotheratom\}}=\{\anatom,\anotheratom\}$
 and $\atoms{\{\{\anatom\},\{\logicalnot\anotheratom\}\}}=\{\anatom,\anotheratom\}$.
 For a set~${\astringofliterals}$ of literals, by
$\answersetof{\astringofliterals}$ we denote atoms that occur positively
in ${\astringofliterals}$. For instance,
$\answersetof{\{\anatom,\logicalnot\anotheratom\}}=\{\anatom\}$.
For a set $\asetofatoms$ of atoms and a set~$\astringofliterals$ of literals, by
$\restriction{\astringofliterals}{\asetofatoms}$ we denote the maximal subset of
$\astringofliterals$ over $\asetofatoms$.
For example,
$\restriction{\{\anatom,\logicalnot\anotheratom,\athirdatom\}}{\{\anatom,\anotheratom\}}
=\{\anatom,\logicalnot\anotheratom\}$.

A {\em (truth) assignment} to a set $\asetofatoms$ of atoms is a function
 from $\asetofatoms$ to $\{\false,\true\}$.
 An assignment {\em satisfies} a formula $\aformula$ if $\aformula$
 evaluates to $\true$ under this assignment. We call an assignment that satisfies formula $F$
 a {\em satisfying assignment} or a
 {\em (classical) model} for
 $\aformula$.
If $\aformula$
 evaluates to $\false$ under an assignment, we say that this assignment {\em contradicts}
 $\aformula$. If $\aformula$ has no model we say that $\aformula$
is {\em unsatisfiable}.
For sets $\asetofatoms$ and~$\anothersetofatoms$ of atoms such that
$\asetofatoms\subseteq\anothersetofatoms$, we identify $\asetofatoms$
 with an
assignment to $\anothersetofatoms$ as follows:
if $\anatom\in \asetofatoms$ then $\anatom$ maps to $\true$, while
if $\anatom\in\anothersetofatoms\setminus \asetofatoms$ then $\anatom$ maps to
$\false$.
We also identify a consistent set $\astringofliterals$ of
literals (i.e., a set that does not contain both a literal and its complement)
with an assignment to $\atoms{\astringofliterals}$ as follows:
if $\anatom\in \astringofliterals$ then $\anatom$ maps to $\true$, while
if $\logicalnot\anatom\in \astringofliterals$ then $\anatom$ maps to
$\false$. The set $\asetofliterals$ is a complete set of literals over the set
of atoms $\asetofatoms$ if $\atoms{\asetofliterals}=\asetofatoms$;
hence a consistent and complete set of literals over $\asetofatoms$
represents an assignment to $\asetofatoms$.

%

\paragraph{Logic Programs.}
A {\em head} is a (possibly empty) disjunction of atoms.
A {\em body} is an expression of the form
\beq
\anatom_{1},\dots, \anatom_\anotherindex,
 \aspnot \anatom_{\anotherindex+1},\dots,\aspnot \anatom_\athirdindex
 \eeq{e:body}
where $\anatom_{1},\dots,\anatom_\athirdindex$ are atoms, and $\aspnot$ is the negation-as-failure operator.
We identify body~\eqref{e:body} with the following conjunction of literals
\[
\anatom_{1}\logicaland \dots\logicaland \anatom_\anotherindex\logicaland
 \logicalnot \anatom_{\anotherindex+1}\logicaland \dots\logicaland \logicalnot
 \anatom_\athirdindex.
\]
Expressions $\anatom_{1},\dots, \anatom_\anotherindex$ and $\aspnot \anatom_{\anotherindex+1},\dots,\aspnot \anatom_\athirdindex$
are called {\em positive} and {\em negative} parts of the body, respectively.
Recall that we sometimes view a conjunction of literals as a set containing all
of its literals.
Thus, given body $B$ we may write an expression $b\in B$, which
means that atom $b$ occurs in the positive part of the body.
Similarly, an expression $\neg b\in B$
 means that the atom $b$ (or, in other words, expression $\aspnot b$) occurs in the negative part of the body.

A {\em disjunctive rule} is an expression of the form
$\ahead\aspimplication\abody$, where $\ahead$ is a head and $\abody$ is a body.
If $\ahead$ is empty we drop it from the expression.
A \emph{disjunctive logic program} is a finite set of
 {\em disjunctive rules}.
We call a rule {\em non-disjunctive} if its head contains no
more than one atom. A program is
{\em non-disjunctive} if it consists of
non-disjunctive rules. By $\atoms{\aprogram}$ we denote the set of atoms occurring in a logic program $\aprogram$.
If we understand $\ahead\aspimplication \abody$ as a
classical logic implication, we can see any rule
$\ahead\aspimplication\abody$ as logically equivalent to clause
$\ahead\logicalor\opposite{\abody}$~(if $A$ is an empty clause then we view the rule as the clause
$\opposite{\abody}$).
This allows us to view a program $\aprogram$ as a CNF formula when useful.
Conversely, we identify CNF formulas with logic programs: syntactically,
every clause $\aclause$ in a given formula is seen as a rule
$\aspimplication\aclause$. For instance
$\anatom_1\logicalor\logicalnot\anatom_2$
is seen as a rule $\aspimplication \aspnot \anatom_1, \anatom_2$.

The presented definition of a logic program accounts for propositional programs only. Indeed, all modern disjunctive answer set solvers  consider propositional programs only. In practice, answer set programmers devise programs with variables. 
Software systems called  grounders~\cite{syrj-2001,PerriSCL07} are used to take a logic
program with variables as its input and produce a propositional
program as its output so that the resulting propositional program has the same answer sets as the input program.

\paragraph{Reduct and Supporting Rules.}
In the following definition we write rules in the form
$\ahead\aspimplication \abody_1,\abody_2$
where $\abody_1$ denotes the positive part of the body, whereas
$\abody_2$ denotes the negative part of the body.
The {\em reduct} $\reduct{\aprogram}{\asetofatoms}$ of a disjunctive program
$\aprogram$ with respect to a set $\asetofatoms$ of atoms is obtained from~$\aprogram$ by
deleting each rule~$\ahead\aspimplication \abody_1,\abody_2$
 such that $\asetofatoms\setintersection\atoms{\abody_2}\neq\emptyset$
and replacing each remaining rule~$\ahead\aspimplication \abody_1,\abody_2$ with
$\ahead\aspimplication \abody_1$.
A set~$X$ of atoms is an {\em answer set} of a program~$\Pi$ if~$X$ is minimal among the sets of atoms that satisfy $\Pi^X$.

For a program $\aprogram$, an atom $\anatom$, and a set $\astringofliterals$ of literals,
we call any rule $\ahead\logicalor\anatom\aspimplication\abody$ in $\aprogram$
a {\em supporting} rule for $\anatom$ with respect to $\astringofliterals$ when
$\astringofliterals\cap(\opposite{\abody}\cup{\ahead})=\emptyset$.

A consistent and complete set $\astringofliterals$ of literals over
$\atoms{\aprogram}$ is
\begin{enumerate}\itemsep0pt\parskip0pt\parsep0pt
\item a {\em classical model} of $\aprogram$ if $\astringofliterals$ satisfies every rule
      in $\aprogram$;
\item a {\em supported model} of $\aprogram$ if $\astringofliterals$
      is a classical model of $\aprogram$ and for every atom
      $\anatom\in\answersetof{\astringofliterals}$
      there is a supporting rule for $\anatom$ with respect to $\astringofliterals$;
\item
      a {\em stable model} of program $\aprogram$ if
$\answersetof{{\astringofliterals}}$
      is an answer set of $\Pi$.
\end{enumerate}

\paragraph{Completion.}
The {\em completion} $\completion{\aprogram}$ of a program $\aprogram$ is
the formula that consists of $\aprogram$ and the formulas 
\beq
\{ \logicalnot \anatom\logicalor
\bigvee_{\ahead\logicalor\anatom \aspimplication\abody\in\aprogram}
(\abody\logicaland\opposite{\ahead})\mid ~ \anatom\in\atoms\aprogram\}.
\eeq{eq:comp-second}
This formula has the property that any stable model of $\aprogram$
is a classical model of $\completion{\aprogram}$.
The converse
does not hold in general.


For a program $\aprogram$ and
a consistent set $\astringofliterals$ of literals
over $\atoms{\aprogram}$,
a set $\asetofatoms$ of atoms over $\atoms{\aprogram}$ is said to be \emph{unfounded}~\cite{leo97} on
$\astringofliterals$ with respect to the program $\aprogram$ when for each atom $\anatom\in \asetofatoms$ and
 each rule
$\ahead\aspimplication\abody\in\aprogram$ such that
$\anatom\in\ahead$, either of the following conditions hold
\begin{enumerate}\itemsep0pt\parskip0pt\parsep0pt
\item $\astringofliterals\cap\opposite{\abody}\neq\emptyset$,
\item $\asetofatoms\cap\abody\neq\emptyset$, or
\item $(\ahead\setminus\asetofatoms) \setintersection \astringofliterals\neq\emptyset$.
\end{enumerate}

We restate Theorem~4.6 from \citeN{leo97} that relates the notions of unfounded
set and stable model.
\begin{theorem}\label{thm:leone}
For a program~$\aprogram$ and
 a consistent and complete set $\astringofliterals$ of literals over
$\atoms{\aprogram}$, $\astringofliterals$ is a stable model of $\aprogram$ if
and only if $\astringofliterals$ is a classical model of $\aprogram$ and no
non-empty subset of $\answersetof{\astringofliterals}$ is an unfounded set
on~$\astringofliterals$ with respect to $\aprogram$.
\end{theorem}
This theorem is crucial for understanding key
computational ideas behind modern answer set solvers.


\subsection{Abstract \dpll}
The Davis--Putnam--Logemann--Loveland ({\dpll}) algorithm from~\citeN{dav62} is
a well-known method that exhaustively explores sets of literals to generate classical
models of a propositional formula.
Most satisfiability and non-disjunctive answer set solvers are based on variations of
the {\dpll} procedure that is a classical backtrack search-based algorithm.
We now review the abstract transition system for {\dpll} proposed by
\citeN{nie06}, which is an
alternative to common pseudo-code descriptions of backtrack search-based algorithms.
For our purposes it is convenient to state {\dpll} as the procedure applied to a logic program
in order to find its classical models.



For a set $\asetofatoms$ of atoms,
a \emph{record} relative to~$\asetofatoms$ is
a string $\astringofliterals$ composed of literals over $\asetofatoms$ or the
symbol~$\logicalfalse$ so that there are no repetitions, and some
literals $\aliteral$ may be annotated as $\aliteral^{\decisionsuperscript}$.
The annotated literals are called \emph{decision} literals.
Figure~\ref{figure:basicstates} presents the set of all records relative
to the singleton set~$\{\anatom\}$.
We say that a record $\astringofliterals$ is {\em inconsistent} if it
contains both a literal $\aliteral$ and its complement $\opposite{\aliteral}$, or
if it contains~$\logicalfalse$, and {\em consistent} otherwise.
For instance,
only five records in Figure~\ref{figure:basicstates}, namely $\emptyset$,
$\anatom$, $\logicalnot \anatom$, $\decision{\anatom}$ and
$\decision{\logicalnot\anatom}$, are consistent.
We will sometime
view a record as the set containing
all its elements disregarding their annotations. For example,
a record $\decision{\anotheratom}~\logicalnot \anatom$ is identified with the set
$\{\logicalnot \anatom,\anotheratom\}$.
A \emph{basic state} relative to~$\asetofatoms$
is either \begin{enumerate}\itemsep0pt\parskip0pt\parsep0pt
\item a record relative to~$\asetofatoms$,
\item $\terminalstate{\astringofliterals}$ where $\astringofliterals$ is a
      record relative to~$\asetofatoms$, or
\item the distinguished
      state ${\failstate}$.
\end{enumerate}

\begin{figure}
{\footnotesize
$$
\begin{array}{c}
\emptyset,\ \ \bot,\ \
\anatom,\ \ \logicalnot \anatom, \ \ \decision{\anatom},\ \
  \decision{\logicalnot\anatom},\ \
\anatom~\bot,\ \ \bot~\anatom,\ \ \decision{\anatom}~\bot,\ \
  \bot~\decision{\anatom},\ \
\logicalnot\anatom~\bot,\ \ \bot~\logicalnot\anatom,\ \
  \decision{\logicalnot\anatom}~\bot,\ \ \bot~\decision{\logicalnot\anatom},\ \
\\
\anatom~\logicalnot\anatom,\ \
  \decision{\anatom}~\logicalnot\anatom,\ \
  \anatom~\decision{\logicalnot\anatom},\ \
  \decision{\anatom}~\decision{\logicalnot\anatom},\ \
\bot~\anatom~\logicalnot\anatom,\ \
  \bot~\decision{\anatom}~\logicalnot\anatom,\ \
  \bot~\anatom~\decision{\logicalnot\anatom},\ \
  \bot~\decision{\anatom}~\decision{\logicalnot\anatom},\ \
  \\
\anatom~\logicalnot\anatom~\bot,\ \
  \decision{\anatom}~\logicalnot\anatom~\bot,\ \
  \anatom~\decision{\logicalnot\anatom}~\bot,\ \
  \decision{\anatom}~\decision{\logicalnot\anatom}~\bot,\ \
\anatom~\bot~\logicalnot\anatom,\ \
  \decision{\anatom}~\bot~\logicalnot\anatom,\ \
  \anatom~\bot~\decision{\logicalnot\anatom},\ \
  \decision{\anatom}~\bot~\decision{\logicalnot\anatom},\ \
\\
\logicalnot\anatom~\anatom,\ \
  \decision{\logicalnot\anatom}~\anatom,\ \
  \logicalnot\anatom~\decision{\anatom},\ \
  \decision{\logicalnot\anatom}~\decision{\anatom},\ \
\bot~\logicalnot\anatom~\anatom,\ \
  \bot~\decision{\logicalnot\anatom}~\anatom,\ \
  \bot~\logicalnot\anatom~\decision{\anatom},\ \
  \bot~\decision{\logicalnot\anatom}~\decision{\anatom},\ \
  \\
\logicalnot\anatom~\anatom~\bot,\ \
  \decision{\logicalnot\anatom}~\anatom~\bot,\ \
  \logicalnot\anatom~\decision{\anatom}~\bot,\ \
  \decision{\logicalnot\anatom}~\decision{\anatom}~\bot,\ \
\logicalnot\anatom~\bot~\anatom,\ \
  \decision{\logicalnot\anatom}~\bot~\anatom,\ \
  \logicalnot\anatom~\bot~\decision{\anatom},\ \
  \decision{\logicalnot\anatom}~\bot~\decision{\anatom}.\ \
\end{array}
$$
}\normalsize
\caption{Records relative to $\{\anatom\}$.}\label{figure:basicstates}
\end{figure}


Each program $\aprogram$ determines its
\emph{\dpll graph} $\ordinarydp{\aprogram}$.
The set of nodes of $\ordinarydp{\aprogram}$ consists of the basic states
relative to~$\atoms{\aprogram}$.
A node in the graph is \emph{terminal} if no edge originates from
it. The state $\emptyset$ is called {\em initial}.
The edges of the graph $\ordinarydp{\aprogram}$ are
specified by the transition rules presented in Figure~\ref{figure:ordinarydp}.

Intuitively, every state of the {\dpll graph} represents some hypothetical
state of the {\dpll} computation whereas a path in the graph
is a description of a process of search for a classical model of a given
program.
The rule $\unitdpll$ asserts that we can add a literal that is a logical
consequence of our previous decisions and the given program.
The rule $\decidedpll$ asserts that we make an arbitrary decision to add a
literal or, in other words, to assign a value to an atom.
Since this decision is arbitrary, we are allowed to backtrack at a later point.
The rule $\backtrackdpll$ asserts that the present state of computation is inconsistent but can
be fixed: at some point in the past we added a decision literal
whose value we can now reverse.
The rule $\faildpll$ asserts that the current state of computation has failed and cannot
be fixed.
The rule $\okruledpll$ asserts that the current state of computation
corresponds to a successful outcome.

We say that a graph $\agraph$ {\em checks} a set $\asetofmodels$ of sets of literals
when all the following conditions hold:
\begin{enumerate}\itemsep0pt\parskip0pt\parsep0pt
  \item $\agraph$ is finite and acyclic;
  \item Any terminal state in $\agraph$ is either $\failstate$ or of the form
        $\terminalstate{\astringofliterals}$;
  \item If a state $\terminalstate{\astringofliterals}$ is reachable from
        the initial state in $\agraph$ then
        $\restriction{\astringofliterals}{\atoms{\asetofmodels}}\in\asetofmodels$;
  \item $\failstate$ is reachable from the initial state in $\agraph$
        if and only if $\asetofmodels$ is empty.

\end{enumerate}

\begin{figure}[t]
\footnotesize{
$$
\begin{array}{llll}
\faildpll:&
~~~~~~~  \astringofliterals
    \transitionarrow \failstate &
      \hbox{if~}
        \left\{
          \begin{array}{l}
            \astringofliterals \textrm{ is inconsistent }
            \textrm{ and }\\
            \astringofliterals \textrm{ contains no decision literals }
          \end{array}
        \right. \\ \\
\backtrackdpll:&
~~~~~~~  \astringofliterals\decision{\aliteral}\astringofliterals'
    \transitionarrow\astringofliterals\opposite{\aliteral} &
      \hbox{if~}
        \left\{
          \begin{array}{l}
            \astringofliterals\decision{\aliteral}\astringofliterals'
            \textrm{ is inconsistent }
            \textrm{ and }\\
            \astringofliterals'
            \textrm{ contains no  decision literals }
          \end{array}
        \right.\\ \\
\unitdpll:\hspace{-5pt} &
~~~~~~~  \astringofliterals
    \transitionarrow \astringofliterals \aliteral &
      \hbox{if~}
         \left\{
           \begin{array}{l}
             \aliteral\textrm{ does not occur in }\astringofliterals
             \textrm{ and }\\
             \textrm{a rule in }\aprogram
                \textrm{ is equivalent to }
                \aclause \logicalor \aliteral
             \textrm{ and}\\
             \textrm{all the literals of } \opposite{\aclause}
               \textrm{ occur in } \astringofliterals
           \end{array}
         \right.\\ \\
\decidedpll: &
~~~~~~~  \astringofliterals
    \transitionarrow\astringofliterals\decision{\aliteral} &
      \hbox{if~}
        \left\{
          \begin{array}{l}
            \astringofliterals \textrm{ is consistent and }\\
            \textrm{neither } \aliteral \textrm{ nor } \opposite{\aliteral}
              \textrm{ occur in } \astringofliterals
          \end{array}
        \right.\\ \\
\okruledpll:&
~~~~~~~  \astringofliterals
    \transitionarrow\terminalstate{\astringofliterals} &
      \hbox{if no other rule applies }\\
\end{array}
$$
}
\normalsize
\caption{Transitions of the graph
$\ordinarydp{\aprogram}$.}\label{figure:ordinarydp}
\end{figure}

\begin{proposition}~\label{prop:dp}
For any program $\aprogram$, the graph $\ordinarydp{\aprogram}$ checks
the classical models of $\aprogram$.
\end{proposition}

Thus, to decide the satisfiability of a program~$\aprogram$ it is enough
to find a path leading from node $\emptyset$ to a terminal node.
 If it is
$\failstate$, then $\aprogram$ has no classical models. Otherwise, $\aprogram$ has classical models.
For instance, let $\aprogram_1$ be
\[
\ba{l}
\aspimplication\aspnot\anatom,~
\aspnot\anotheratom\\ \aspimplication \anatom,~ \aspnot \athirdatom.
\ea
\]
Figure~\ref{figure:transitionexample} presents two
paths in $\ordinarydp{\aprogram_1}$ from the node $\emptyset$ to the node
$\terminalstate{\decision{\anatom}~\athirdatom~\decision{\anotheratom}}$.
Every edge is annotated on the left by
the name of the transition rule that gives rise to this edge in $\ordinarydp{\aprogram_1}$.
The node
$\terminalstate{\decision{\anatom}~\athirdatom~\decision{\anotheratom}}$ is
terminal. Thus, Proposition~\ref{prop:dp} asserts that $\aprogram_1$ is
satisfiable and $\seenasamodel{\{\anatom,\athirdatom,\anotheratom\}}$ is a
classical model of~$\aprogram_1$.

\begin{figure}
\footnotesize{
$$
\begin{array}{c|c}
\begin{array}{lll}
\textrm{Initial state :}
   &
   & \emptyset
   \\
\decidedpll
   & \transitionarrow
   & \decision{\anatom}
   \\
\unitdpllshort
   & \transitionarrow
   & \decision{\anatom}~\athirdatom
   \\
\decidedpll
   & \transitionarrow
   & \decision{\anatom}~\athirdatom~\decision{\anotheratom}
   \\
\okruledpll
   & \transitionarrow
   & \terminalstate{\decision{\anatom}~\athirdatom~\decision{\anotheratom}}

\end{array}\label{eq:path1}
&
\begin{array}{lll}
\textrm{Initial state :}
   &
   & \emptyset
   \\
\decidedpll
   & \transitionarrow
   & \decision{\anatom}
   \\
\decidedpll
   & \transitionarrow
   & \decision{\anatom}~\decision{\logicalnot\athirdatom}
   \\
\unitdpllshort
   & \transitionarrow
   & \decision{\anatom}~\decision{\logicalnot\athirdatom}~\athirdatom
   \\
\backtrackdpll
   & \transitionarrow
   & \decision{\anatom}~\athirdatom
   \\
\decidedpll
   & \transitionarrow
   & \decision{\anatom}~\athirdatom~\decision{\anotheratom}
   \\
\okruledpll
   & \transitionarrow
   & \terminalstate{\decision{\anatom}~\athirdatom~\decision{\anotheratom}}

\end{array}\label{eq:path2}
\end{array}
$$
}
\normalsize
\caption{Examples of paths in
$\ordinarydp{{\{\aspimplication\aspnot\anatom,~\aspnot\anotheratom;~~
\aspimplication\anatom,~\aspnot\athirdatom\}}}$.}
\label{figure:transitionexample}
\end{figure}

A path in the graph $\ordinarydp{\aprogram}$
is a description of a process of search for a classical model
of a program $\aprogram$. The process is captured via applications of
transition rules.
Therefore, we can characterize the algorithm
of a solver that utilizes the transition rules of
$\ordinarydp{\aprogram}$
by describing a strategy for choosing a path.
A strategy can be based on assigning priorities to
transition rules of $\ordinarydp{\aprogram}$
so that a solver never applies a rule in a node
if a rule with higher priority is applicable
to the same node. The {\dpll} procedure is captured by the
priorities ordered as we stated rules in Figure~\ref{figure:ordinarydp}.
For instance, transition rule~$\faildpll$ has the highest priority.
In Figure~\ref{figure:transitionexample}, the path on the left complies with the
{\dpll} priorities: Thus, it corresponds to an execution of the {\dpll}
procedure. The path on the right does not: it uses $\decidedpll$ when $\unitdpll$ is applicable.
The proof of Proposition~\ref{prop:dp} follows the lines of the proof of
Theorem 2.13 in \citeN{nie06}\footnote{This
work defines a different \dpll graph, avoiding the reference to the transition
rule $\okruledpll$.
The presence of this rule in this presentation is important for the
generalizations of the \dpll graph we introduce in the sequel.}.

\paragraph{Abstract Answer Set Solver for Non-disjunctive Programs.}
\citeN{lier10} illustrated that extending~$\ordinarydp{\aprogram}$ by
a transition rule \footnotesize{
$$
\begin{array}{llll}
\unfoundeddpll:\hspace{-5pt} &
~~~~~~~  \astringofliterals
    \transitionarrow \astringofliterals~ \logicalnot\anatom &
      \hbox{if~}
         \left\{
         \begin{array}{l}
          \logicalnot\anatom\textrm{ does not occur in }\astringofliterals
                        \textrm{ and}\\

          \astringofliterals \textrm{ is consistent}
              \textrm{ and }\\
            \textrm{there is a set } \asetofatoms \textrm{ of atoms
               containing } \anatom \textrm{ such that }\\
            \asetofatoms \textrm{ is unfounded on }
              \astringofliterals \textrm{ w.r.t. } \aprogram
              \end{array}
         \right.\\ \\
\end{array}
$$
}
\normalsize

\noindent
captures a backtrack-search procedure for finding answer sets of
non-disjunctive programs. Many answer set solvers for such programs can
be seen as extensions of this procedure~\cite{lie11a}.

\section{A Two-Layer Abstract Solver}
\label{sec:cm}
%
%



The problem of
deciding whether a disjunctive program has a stable model is
$\Sigma^P_2$-complete~\cite{eit93a}.
This translates into the following:
~(i)~there is an exponential number of
possible candidate models, and~(ii)~the problem of
deciding whether a
candidate model is an answer set
 of a disjunctive logic program is co-NP-complete.
The latter condition differentiates algorithms of answer set solvers for
disjunctive programs from the procedures for non-disjunctive programs.
Indeed, the problem of
deciding whether a
candidate model is an answer set
 of a non-disjunctive program is tractable.

A common architecture of a disjunctive answer set solver is composed of two layers corresponding to the two above conditions: a generate layer and a test
layer, each typically based on {\dpll}-like procedures. In particular:
\begin{itemize}\itemsep0pt\parskip0pt\parsep0pt
 \item 
  The generate layer is used to
  obtain a set of candidates that are potentially stable models.
 \item 
  The test layer is used to verify whether a candidate (produced by the generate layer) is a stable model of
  the given program.
\end{itemize}

We now proceed to present a graph $\firstgraph{\themainprogram}{\thedependingprogram}(\aprogram)$
that captures such
two-layer architecture.
It is based on instances of the {\dpll} procedure
for both its generating task and its testing task.
We then illustrate how the
$\firstgraph{\themainprogram}{\thedependingprogram}(\aprogram)$ transition system can
be used to capture the disjunctive answer set solver {\cmodels} in its basic
form.

\subsection{A Two-Layer Abstract Solver via \dpll}

\begin{figure}[t]
\footnotesize{
$$
\arraycolsep=2pt
\begin{array}{llll}
\multicolumn{4}{l}{\textrm{Left-rules}}\\
\failleft
  & (\astringofliterals,
    \emptyset)_\leftstate
     & \transitionarrow \failstate
        & \textrm{if}\left\{
          \begin{array}{l}
          \astringofliterals \textrm{ is inconsistent}
                       \textrm{ and}\\
          \astringofliterals\textrm{ contains no decision literal}
          \end{array}
          \right.\\ \\

\backtrackleft
  & (\astringofliterals\decision{\aliteral}\astringofliterals',
     \emptyset)_\leftstate
     & \transitionarrow
       (\astringofliterals\opposite{\aliteral},
       \emptyset)_\leftstate
        & \textrm{if}\left\{
          \begin{array}{l}
          \astringofliterals\decision{\aliteral}\astringofliterals'
                        \textrm{ is inconsistent}
                        \textrm{ and}\\
          \astringofliterals'\textrm{ contains no decision literal}
          \end{array}
          \right.\\ \\

\unitleft
  & (\astringofliterals,
    \emptyset)_\leftstate
     & \transitionarrow
       (\astringofliterals\aliteral,
       \emptyset)_\leftstate
        & \textrm{if}\left\{
          \begin{array}{l}
          \aliteral\textrm{ is a literal over }
                        \atoms{\mainprogram{\aprogram}}
                        \textrm{ and}\\
          \aliteral\textrm{ does not occur in }\astringofliterals
                        \textrm{ and}\\
          \textrm{a rule in }\mainprogram{\aprogram}
                        \textrm{ is equivalent to }
                        \aclause \logicalor \aliteral
                        \textrm{ and}\\
          \textrm{all the literals of } \opposite{\aclause}
                        \textrm{ occur in } \astringofliterals
          \end{array}
          \right.\\ \\

\decideleft
  & (\astringofliterals,
    \emptyset)_\leftstate
     & \transitionarrow
       (\astringofliterals\decision{\aliteral},
       \emptyset)_\leftstate
        & \textrm{if}\left\{
          \begin{array}{l}
          \astringofliterals \textrm{ is consistent}
                       \textrm{ and}\\
          \aliteral\textrm{ is a literal over }\atoms{\mainprogram{\aprogram}}
                       \textrm{ and}\\
          \textrm{neither }\aliteral\textrm{ nor }
                       \opposite{\aliteral}
                       \textrm{ occur in }\astringofliterals
          \end{array}
          \right.\\
\\

\multicolumn{4}{l}{\textrm{Crossing-rule } \leftstate\rightstate}\\
\crossrule
  & (\astringofliterals,
     \emptyset)_\leftstate
     & \transitionarrow
       (\astringofliterals,
       \emptyset)_\rightstate
        & \textrm{if}\left\{
          \begin{array}{l}
          \textrm{no left-rule applies}
          \end{array}
          \right.\\ \\
\\

\multicolumn{4}{l}{\textrm{Right-rules}}\\
\failright
  & (\astringofliterals,
     \anotherstringofliterals)_\rightstate
     & \transitionarrow \terminalstate{\astringofliterals}
        & \textrm{if}\left\{
          \begin{array}{l}
          \anotherstringofliterals \textrm{ is inconsistent}
                       \textrm{ and}\\
          \anotherstringofliterals
                       \textrm{ contains no decision literal}
          \end{array}
          \right.\\ \\

\backtrackright
  & (\astringofliterals,
     \anotherstringofliterals\decision{\aliteral}\anotherstringofliterals')_\rightstate
     & \transitionarrow
       (\astringofliterals,
       \anotherstringofliterals\opposite{\aliteral})_\rightstate
        & \textrm{if}\left\{
          \begin{array}{l}
          \anotherstringofliterals\decision{\aliteral}\anotherstringofliterals'
                        \textrm{ is inconsistent}
                        \textrm{ and}\\
          \anotherstringofliterals'\textrm{ contains no decision literal}
          \end{array}
          \right.\\ \\

\unitright
  & (\astringofliterals,
     \anotherstringofliterals)_\rightstate
     & \transitionarrow
       (\astringofliterals,
       \anotherstringofliterals\aliteral)_\rightstate
        & \textrm{if}\left\{
          \begin{array}{l}
          \aliteral\textrm{ is a literal over }
                        \atoms{\dependingprogram{\aprogram}{\astringofliterals}}
                        \textrm{ and}\\
          \aliteral\textrm{ does not occur in }\anotherstringofliterals
                        \textrm{ and}\\
          \textrm{a rule in }\dependingprogram{\aprogram}{\astringofliterals}
                        \textrm{ is equivalent to }
                        \aclause \logicalor \aliteral
                        \textrm{ and}\\
          \textrm{all the literals of } \opposite{\aclause}
                        \textrm{ occur in } \astringofliterals
          \end{array}
          \right.\\ \\

\decideright
  & (\astringofliterals,
     \anotherstringofliterals)_\rightstate
     & \transitionarrow
       (\astringofliterals,
       \anotherstringofliterals\decision{\aliteral})_\rightstate
        & \textrm{if}\left\{
          \begin{array}{l}
          \anotherstringofliterals \textrm{ is consistent}
                       \textrm{ and}\\
          \aliteral\textrm{ is a literal over }\atoms{\dependingprogram{\aprogram}{\astringofliterals}}
                       \textrm{ and}\\
          \textrm{neither }\aliteral\textrm{ nor }
                       \opposite{\aliteral}
                       \textrm{ occur in }\anotherstringofliterals
          \end{array}
          \right.\\
\\

\multicolumn{4}{l}{\textrm{Crossing-rules } \rightstate\leftstate }\\
\failcross
  & (\astringofliterals,
     \anotherstringofliterals)_\rightstate
     & \transitionarrow \failstate
        & \textrm{if}\left\{
          \begin{array}{l}
          \textrm{no right-rule applies}
             \textrm{ and }\\
          \astringofliterals \textrm{ contains no decision literal}
          \end{array}
          \right.\\ \\

\backtrackcross
  & (\astringofliterals\decision{\aliteral}\astringofliterals',
     \anotherstringofliterals)_\rightstate
     & \transitionarrow
       (\astringofliterals \opposite{\aliteral},
       \emptyset)_\leftstate
        & \textrm{if}\left\{
          \begin{array}{l}
          \textrm{no right-rule applies}
             \textrm{ and }\\
          \astringofliterals'\textrm{ contains no decision literal}
          \end{array}
          \right.\\
\end{array}
$$
}
\caption{The transition rules of the graph
$\firstgraph{\themainprogram}{\thedependingprogram}(\aprogram)$.}
\label{fig:trfirst}
\end{figure}
\normalsize

We start by extending the notion of a basic state to accommodate for generate and test layers.
We call symbols $\leftstate$ and $\rightstate$ {\em labels}.
A {\em state} relative to sets~$\asetofatoms$ and~$\asetofatoms'$ of atoms
 is either
\begin{enumerate}[noitemsep]
 \item a pair $(\astringofliterals,\anotherstringofliterals)_\aside$,
       where $\astringofliterals$ and~$\anotherstringofliterals$ are records
       relative to~$\asetofatoms$ and $\asetofatoms'$, respectively,
       and $\aside$ is a label (either symbol $\leftstate$ or $\rightstate$),
 \item $\terminalstate{\astringofliterals}$, where $\astringofliterals$ is a
       record relative to~$\asetofatoms$, or
 \item the distinguished state $\failstate$.
\end{enumerate}

We say that a set $\asetofliterals$ of literals {\em covers} a program $\aprogram$ if
$\atoms{\aprogram}\subseteq\atoms{\asetofliterals}$.
We say that a function $\themainprogram$
from a program to another program is a {\em generating (program)} function if 
 for any program $\aprogram$,
$\atoms{\aprogram}\subseteq\atoms{\themainprogram(\aprogram)}$.
We call a function
from a program~$\aprogram$ and a consistent set $\asetofliterals$ of
literals covering $\aprogram$ to a non-disjunctive
program~$\aprogram'$ a {\em witness (program)} function.
Intuitively, a program $\aprogram'$ resulting from a
witness function is a {\em witness (program)} with respect to~$\aprogram$ and $\amodel$.
For a program $\aprogram$ and a witness function $\thedependingprogram$,
by $\atoms{\thedependingprogram,\aprogram,\asetofatoms}$ we denote the
union of $\atoms{\dependingprogram{\aprogram}{\astringofliterals}}$ for all
possible consistent and complete sets $\astringofliterals$ of literals over $\asetofatoms$.

We are now ready to define a graph
 $\firstgraph{\themainprogram}{\thedependingprogram}(\aprogram)$ for a
 generating function $\themainprogram$, a witness function
 $\thedependingprogram$ and a program $\aprogram$.
The set of nodes of
$\firstgraph{\themainprogram}{\thedependingprogram}(\aprogram)$ consists of the
states relative to sets $\atoms{\mainprogram{\aprogram}}$ and
$\atoms{\thedependingprogram,\aprogram,\atoms{\mainprogram{\aprogram}}}$.
The state $(\emptyset,\emptyset)_\leftstate$ is called {\em initial}.
The edges of the graph $\firstgraph{\themainprogram}{\thedependingprogram}(\aprogram)$
 are specified by the transition rules presented in Figure~\ref{fig:trfirst}.
The graph $\firstgraph{\themainprogram}{\thedependingprogram}(\aprogram)$
can be used for deciding whether a program $\mainprogram{\aprogram}$
has a classical model $\amodel$ such that the witness
$\dependingprogram{\aprogram}{\amodel}$ is unsatisfiable.

\begin{proposition}~\label{thm:firstgraph}
For any generating function $\themainprogram$, any witness function
$\thedependingprogram$ and any program $\aprogram$, the graph
$\firstgraph{\themainprogram}{\thedependingprogram}(\aprogram)$ checks the classical models
$\amodel$ of $\mainprogram{\aprogram}$ such that
$\dependingprogram{\aprogram}{\amodel}$ is unsatisfiable.
\end{proposition}

\paragraph{Informal Account of the Two-Layer Abstract Solver.}
Each of the rules of the graph
$\firstgraph{\themainprogram}{\thedependingprogram}(\aprogram)$ is placed into
one of the three groups {\em Left}, {\em Right}, and {\em Crossing}.
The left-rules of $\firstgraph{\themainprogram}{\thedependingprogram}(\aprogram)$
capture the generate layer that applies the {\dpll} procedure to the program
$\mainprogram{\aprogram}$ produced by the generating function.
The right-rules of
$\firstgraph{\themainprogram}{\thedependingprogram}(\aprogram)$ capture the
test layer that applies the {\dpll} procedure to the computed witness program.
The label $\leftstate$ (resp. $\rightstate$) suggests that currently the
computation is within the generate (resp. test) layer.
The left-hand-side $\astringofliterals$ (resp. right-hand-side
$\anotherstringofliterals$) of the state
$(\astringofliterals,\anotherstringofliterals)_\leftstate$ records the 
 computation state due to the generate (resp. test) layer.
The crossing rules form a bridge between the two layers.

It turns out that the left-rules no
longer apply to a state of the form $(\astringofliterals,\anotherstringofliterals)_\leftstate$
only when $\astringofliterals$ is a classical model of
$\mainprogram{\aprogram}$.
Thus, when a classical model $\astringofliterals$ of $\mainprogram{\aprogram}$
is found, then the $\crossrule$ is used and a witness program with respect to
$\astringofliterals$ is computed.
If no classical model is found for the witness program, then
$\failright$ rule applies, which brings us to a terminal state
$\terminalstate{\astringofliterals}$, suggesting that
$\astringofliterals$
represents a solution to a given search problem.
It turns out that no right-rules applies in a state of the form
$(\astringofliterals,\anotherstringofliterals)_\rightstate$ only when
$\anotherstringofliterals$ is a classical model for the witness program. Thus,
the set $\astringofliterals$ of literals is not such that
$\dependingprogram{\aprogram}{\amodel}$ is unsatisfiable
and the {\dpll} procedure of the generate layer, embodied by the left-rules,
proceeds with the search, after backtracking through $\backtrackcross$. In the
case when $\backtrackcross$ cannot be applied, it follows that
no other candidate can be found by the generate layer,
so the transition $\failcross$ leading to
 $\failstate$ is the only available one from such a state.


\subsection{Abstract basic {\cmodels}}
\label{sec:cmodelswithoutlearning}
We now relate the  graph
$\firstgraph{\themainprogram}{\thedependingprogram}(\aprogram)$ to the
procedure {\sc dp-assat-proc} from~\citeN{lie05}. This procedure forms the basis
of the answer set solver {\cmodels}. Yet, it does not account for backjumping and
learning techniques, implemented in {\cmodels}.

Given a disjunctive program $\aprogram$, the answer set solver {\sc cmodels}
starts its computation by computing a CNF formula
$\gencmodels{\aprogram}$ that corresponds to the clausified program completion 
of $\aprogram$.
The {\dpll} procedure is then applied to $\gencmodels{\aprogram}$.
The test layer of the {\cmodels} computation relies on the programs produced
by a witness program function called $\thetestcmodels$ that intuitively tests
minimality of found models of completion.

To be complete in our presentation, we now review the details of
$\thegencmodels$ and $\thetestcmodels$ functions~\cite{lierphd}.
To construct $\gencmodels{\aprogram}$, \cmodels introduces an
auxiliary atom $\auxvariable{\abody}$ for every body~$\abody$ occurring in
$\aprogram$.
The atom $\auxvariable{\abody}$ is an explicit definition for
$\abody$, it is true if and only if $\abody$ is true.
Also every disjunctive rule gives rise to as many auxiliary variables as there are atoms in the head of the rule:
for a disjunctive rule $\ahead\aspimplication\abody$
and every atom $\anatom\in \ahead$, an auxiliary atom
$\auxvariable{\anatom,\abody}$ is equivalent to a conjunction
$\abody\wedge\overline{\ahead'}$, where $\ahead'$ is $(\ahead\setminus\{a\})^{\vee}$.
Formulas~\eqref{CmodelsGenEquation} and~\eqref{CmodelsTestEquation}
present the definitions of $\thegencmodels$ and $\thetestcmodels$ for a program
$\Pi$. The first four lines of the definition of the CNF formula
$\gencmodels{\aprogram}$ concern clausification of the introduced explicit
definitions, namely $\auxvariable{\abody}$ and $\auxvariable{\anatom,\abody}$.
The last two lines encode clausified completion with the use of
$\auxvariable{\abody}$ and $\auxvariable{\anatom,\abody}$.
\begin{equation}
\label{CmodelsGenEquation}
\begin{array}{lll}
\gencmodels{\aprogram} =
	  &\{\auxvariable{\abody}\logicalor
             {\opposite{\abody}}\mid
                 \abody\in \bodies{\aprogram}\}\\
          & \{\neg \auxvariable{\abody}\logicalor \anatom
             \mid
                 \abody\in \bodies{\aprogram}, \anatom\in\abody\}\\
          &\{\auxvariable{\anatom,\abody}\logicalor
             \neg{\auxvariable{\abody}\vee
             \ahead}
             \mid
                 \ahead\logicalor\anatom\aspimplication\abody\in\aprogram\}\\
          &\{\neg \auxvariable{\anatom,\abody} \logicalor \anotheratom 
             \mid
                 \ahead\logicalor\anatom\aspimplication\abody\in\aprogram, \anotheratom\in\opposite{\ahead}\cup\{\auxvariable{\abody}\}\}\\
          &\{\neg\auxvariable{\abody}\logicalor \ahead
             \mid
                 \ahead\aspimplication\abody\in\aprogram\}\\
          & \displaystyle{\{\neg\anatom
                 \bigvee_{\anatom \aspimplication\abody\in\aprogram}\auxvariable{\abody}
                 \bigvee_{\ahead\logicalor\anatom \aspimplication\abody\in\aprogram} \auxvariable{\anatom,\abody}\}} \\

\end{array}
\end{equation}
\begin{equation}
\begin{array}{lll}
\label{CmodelsTestEquation}
\testcmodels{\aprogram}{\amodel} =
             & \{\disjunctionof{\opposite{\answersetof{\restriction{\amodel}{\atoms{\aprogram}}}}}\}
               \setunion\\
             & \{\logicalnot\anatom
                 \mid\logicalnot\anatom\in\restriction{\amodel}{\atoms{\aprogram}}\}
               \setunion\\
             & \{
             {\opposite{\abody}}\logicalor\ahead\mid
                 \ahead\aspimplication\abody\in\reduct{\aprogram}{\answersetof{\amodel}},
                 \abody\subseteq\amodel\},
\end{array}
\end{equation}

 Intuitively, {\cmodels} uses the program $\gencmodels{\aprogram}$ as an
approximation of $\aprogram$ during the generate-layer computation.
Indeed, any stable model of $\aprogram$ is also a classical model of $\gencmodels{\aprogram}$.
The converse does not always hold.
Thus, classical models of $\gencmodels{\aprogram}$ must be checked.
For a classical model $\amodel$  of $\gencmodels{\aprogram}$,
a program produced by $\testcmodels{\aprogram}{\amodel}$ has no classical models iff $\amodel$ is a
 stable model of $\aprogram$. In fact, any model~$N$ of
 $\testcmodels{\aprogram}{\amodel}$ is such that it satisfies the reduct
 $\reduct{\aprogram}{\answersetof{\amodel}}$, while
 $\answersetof{N}\subset\answersetof{\restriction{\amodel}{\atoms{\aprogram}}}$.
 In such case, $\answersetof{\restriction{\amodel}{\atoms{\aprogram}}}$ is not
 an answer set of $\Pi$ by definition and, consequently, $\amodel$ is not a
 stable model of $\Pi$.


By $\dpgraph{\aprogram}$ we denote the graph
$\firstgraph{\thegencmodels}{\thetestcmodels}(\aprogram)$.
Proposition~\ref{correctnessCM} below illustrates that the graph
$\dpgraph{\aprogram}$ can be used for deciding whether a given  program $\aprogram$
has a stable model, similarly as the graph $\ordinarydp{\aprogram}$ can be used
for deciding whether $\aprogram$ has a classical model.

\begin{proposition}~\label{correctnessCM}
For any program $\aprogram$, the graph
$\dpgraph{\aprogram}$
checks the stable models of $\aprogram$.
\end{proposition}

The graph
$\dpgraph{\aprogram}$ 
captures the
search procedure of {\sc dp-assat-proc} of {\cmodels}.
The {\sc dp-assat-proc} algorithm follows the
priorities on its transition rules as they are ordered in
Figure~\ref{fig:trfirst}.
We often use this convention when describing other procedures in the sequel.

\section{Graph Templates}
\label{sec:NewPresentation}
The differences in design choices of disjunctive answer set solvers
obscure the understanding of their similarities.
In~\citeN{blm14}, 
transition systems exemplified by the graph $\dpgraph{\aprogram}$ 
were used to capture
several disjunctive solvers, namely,
{\cmodels}, {\gnt} and {\dlv} implementing backtracking.
The transitions systems made the similarities that these solvers share explicit.
For example, all solvers are based on a two-layer approach in the spirit of the
{\sc dp-assat-proc} algorithm.
 In this work, we make an additional move towards a unifying framework for
 capturing two-layer methods.
We introduce a graph template that we then use to
encompass disjunctive solvers {\cmodels}, {\gnt} and {\dlv}. 

\subsection{A Single Layer Graph Template}\label{sec:singlelayertemplate}
In the next section we will define a graph template suitable for capturing
two-layer computation of disjunctive answer set solvers. As a step in this
direction, we describe here a simpler graph template that can be used to
capture the {\dpll} procedure by encapsulating the {\dpll} graph. We also show
that this template can encapsulate a graph capturing the
computation underlying the algorithm of answer set solver {\smodels} for
non-disjunctive programs.


\paragraph{Template.}
A function
from a program $\aprogram$ and a set of literals over $\atoms{\aprogram}$ to
a set of literals over $\atoms{\aprogram}$ is called a
{\em propagator condition} or, shortly, {\em p-condition}.
Figure~\ref{fig:prop0} presents four p-conditions, namely, $\theunit$,
$\theallcancel$,
$\thebacktrue$, and $\theunfounded$.
For a set
$\asetofpropagators$ of p-conditions, a program $\aprogram$ and a set~$\asetofliterals$ of literals, by $\outprop{\asetofpropagators}{\aprogram}{\asetofliterals}$ we
denote the set of literals
$\bigcup_{\thecondition\in\asetofpropagators}\acondition{\asetofliterals}{}{\aprogram}$.
Intuitively, if each image through a p-condition is a set of possible outcomes,
this set represents the union of the possible outcomes through
$\asetofpropagators$.

\begin{figure}[t]
\footnotesize{
$$
\begin{array}{ll}
\multicolumn{2}{l}{
\propunitpropagate{\astringofliterals}{\aliteral}{\aprogram}}\\
\qquad
& \textrm{ iff }
          \left\{
          \begin{array}{l}
          \aliteral\textrm{ does not occur in }\astringofliterals
                        \textrm{ and}\\

          \textrm{a rule in } \aprogram
                \textrm{ that is equivalent to }
                \aclause\logicalor\aliteral
                \textrm{ and}\\
             \textrm{all the literals of } \opposite{\aclause}
                 \textrm{ occur in } \astringofliterals
          \end{array}
          \right.\\

\\\multicolumn{2}{l}{
\propallcancel{\astringofliterals}{\logicalnot\anatom}{\aprogram}}\\
& \textrm{ iff }
          \left\{
          \begin{array}{l}
          \logicalnot\anatom\textrm{ does not occur in }\astringofliterals
                        \textrm{ and}\\
\textrm{there is no rule in $\Pi$ supporting $\anatom$ with respect to $\astringofliterals$}
          \end{array}
          \right.\\

\\\multicolumn{2}{l}{
\propbacktrue{\astringofliterals}{\aliteral}{\aprogram}}\\
& \textrm{ iff }
          \left\{
          \begin{array}{l}
          \aliteral\textrm{ does not occur in }\astringofliterals
                        \textrm{ and}\\

          \textrm{there is a rule } \ahead\logicalor\anatom
                \aspimplication \abody
                 \textrm{ in } \aprogram\\
            \textrm{so that (i) }
                        \anatom\in\astringofliterals,
\textrm{and (ii) either  $\opposite{\aliteral}\in\ahead$ or $\aliteral\in\abody$ and},\\
\textrm{(iii) no other rule in $\Pi$ is supporting $\anatom$ with respect to $\astringofliterals$}


          \end{array}
         \right.\\

\\\multicolumn{2}{l}{
\propunfounded{\astringofliterals}{\logicalnot\anatom}{\aprogram}}\\
& \textrm{ iff }
          \left\{
          \begin{array}{l}
          \logicalnot\anatom\textrm{ does not occur in }\astringofliterals
                        \textrm{ and}\\

          \astringofliterals \textrm{ is consistent}
              \textrm{ and }\\
            \textrm{there is a set } \asetofatoms
              \textrm{ of atoms containing } \anatom \textrm{ such that }\\
            \asetofatoms \textrm{ is unfounded on }
              \astringofliterals \textrm{ with respect to } \aprogram
          \end{array}
          \right.\\

\end{array}
$$
}
\caption{Propagator conditions.}
\label{fig:prop0}
\end{figure}

\begin{definition}
Given a a program $\aprogram$ and a
set~$\asetofpropagators$ of p-conditions,
a \emph{\dpll graph template}
$\ordinarydpt{\asetofpropagators}{\aprogram}$ is a graph of which
nodes are the basic states relative to $\atoms{\aprogram}$ and
edges are specified by the transition rules
$\faildpll$, $\backtrackdpll$, $\decidedpll$, $\okruledpll$
presented in Figure~\ref{figure:ordinarydp} and the transition rule
\begin{equation}
\begin{array}{llll}
\propagategen
  & \astringofliterals
     & \transitionarrow
       \astringofliterals\aliteral
        & \textrm{ if ~~~ }
      \aliteral\in \outprop{\asetofpropagators}{\aprogram}{\astringofliterals}.
\end{array}
\label{eq:propagategen}
\end{equation}
\end{definition}


For instance, the
instantiation $\ordinarydpt{\{\theunit\}}{\aprogram}$ of the \dpll graph template
results in the
\dpll graph $\ordinarydp{\aprogram}$.
Indeed, by definition these graphs share the same nodes as well as
their rules $\faildpll$, $\backtrackdpll$, $\decidedpll$, and $\okruledpll$
coincide.
Then, one can see that 
$\propunitpropagate{\astringofliterals}{\aliteral}{\aprogram}$ 
if and only if the
transition rule $\unitdpll$ in $\ordinarydp{\aprogram}$ is applicable in
$\astringofliterals$ and supports the transition to a state
$\astringofliterals\aliteral$, which shows that the $\unitdpll$ rule and the
$\propagategen$ rule coincide when $\asetofpropagators=\{\theunit\}$.

\paragraph{Instantiation.}
We call {\em types} the elements of the set $\alltypes=\{\classic,\supported,\stable\}$.
In the following, by $\classic$-model, $\supported$-model and $\stable$-model
we denote classical, supported, and stable models, respectively.
We also use letter $\atypeofmodel$ to denote a variable over set $\alltypes$ of types.
We say that a set $\asetofpropagators$
of p-conditions is
{\em $\atypeofmodel$-sound}
if for any program~$\aprogram$, for any set $\asetofliterals$ of literals,
and for any $\atypeofmodel$-model $\asetofliterals_1$ of~$\aprogram$ such that $\asetofliterals\subseteq \asetofliterals_1$, it also
holds that
$\outprop{\asetofpropagators}{\aprogram}{\asetofliterals}\subseteq\asetofliterals_1$.
Note that any $\classic$-sound set of p-conditions is $\supported$-sound,
and any $\supported$-sound set of p-conditions is $\stable$-sound.
We say that a set $\asetofpropagators$
of p-conditions is
{\em $\atypeofmodel$-complete}
when for any program~$\aprogram$ and any consistent and
complete set $\asetofliterals$ of literals over $\atoms{\aprogram}$, set
$\asetofliterals$ is a $\atypeofmodel$-model of $\aprogram$ if and only if
$\outprop{\asetofpropagators}{\aprogram}{\asetofliterals}=\emptyset$.
For a type $\atypeofmodel$, we say that a set $\asetofpropagators$ of p-conditions is
\emph{$\atypeofmodel$-enforcing} if $\asetofpropagators$
is both $\atypeofmodel$-sound and $\atypeofmodel$-complete.

Next theorem summarizes properties of several sets of p-conditions:
$$
\begin{array}{rll}
  \uppropagators&=&\{\theunit\}\\ 
  \smdisjpropagators&=&\{\theunit,\theallcancel,\thebacktrue\}\\
  \smpropagators&=&\{\theunit,\theallcancel,\thebacktrue,\theunfounded\}\\
\end{array}
$$

\begin{theorem}\label{thm:typecomplete}
The following statements hold:
\setlength\leftmargin{25pt}
\begin{enumerate}\itemsep0pt\parskip0pt\parsep0pt
 \item The set $\uppropagators$ is $\classic$-enforcing;
 \item All the subsets of $\smdisjpropagators$ that contain
       $\{\theunit,\theallcancel\}$ are
       $\supported$-enforcing; and
 \item All the subsets of $\smpropagators$ that contain
       $\{\theunit,\theunfounded\}$ are
       $\stable$--enforcing.
\end{enumerate}
\end{theorem}

We are now ready to state the main result of this section.
\begin{theorem}\label{thm:minitemplate}
For any program $\aprogram$, any type $\atypeofmodel$, and any
$\atypeofmodel$-enforcing set of p-conditions $\asetofpropagators$,
the graph $\ordinarydpt{\asetofpropagators}{\aprogram}$ checks the
$\atypeofmodel$-models of $\aprogram$.
\end{theorem}

Theorems~\ref{thm:typecomplete} and~\ref{thm:minitemplate} give rise to
families of valid solvers for deciding where classical, supported, or stable
models exist for a program.
For instance, for a non-disjunctive program $\aprogram$, the graph
$\ordinarydpt{\smpropagators}{\aprogram}$ coincides with the graph
{\sc sm}$_\aprogram$~\cite{lier10} that captures computation of answer set
solver {\smodels}~\cite{sim02}.
The graph $\ordinarydpt{\smdisjpropagators}{\aprogram}$ coincides with the
graph {\sc atleast}$_\aprogram$~\cite{lier10} that provides a procedure for
deciding whether a non-disjunctive program has supported models.
For a disjunctive program~$\aprogram$
the same single layer graph $\ordinarydpt{\smpropagators}{\aprogram}$ forms a
 procedure for deciding whether $\aprogram$ has a stable model.
Note, however, that generally
the problem of deciding whether
$\propunfounded{\astringofliterals}{\aliteral}{\aprogram}$ is
{\npcomplexity}-complete for the case when ${\aprogram}$ is disjunctive.

\subsection{A Two-Layer Graph Template}
We extend here the approach of Section \ref{sec:singlelayertemplate} to capture
two-layer methodology of disjunctive solvers.

\begin{figure}[t]
$$
\begin{array}{llll}
\propagateleft
  & (\astringofliterals,
    \emptyset)_\leftstate
     & \transitionarrow
       (\astringofliterals\aliteral,
       \emptyset)_\leftstate
        & \textrm{if }
          \aliteral\in\outprop{\asetofpropagators_\leftstate}{\mainprogram{\aprogram}}{\astringofliterals}
\\
\\

\propagateright
  & (\astringofliterals,
     \anotherstringofliterals)_\rightstate
     & \transitionarrow
       (\astringofliterals,
       \anotherstringofliterals\aliteral)_\rightstate
        & \textrm{if }
          \aliteral\in\outprop{\asetofpropagators_\rightstate}
                              {\dependingprogram{\aprogram}{\astringofliterals}}
                              {\anotherstringofliterals}

\end{array}
$$
\caption{Transition rules of the graph template
$\simpletemplate{\asetofpropagators_\leftstate}
{\asetofpropagators_\rightstate}
{\themainprogram}{\thedependingprogram}
(\aprogram)$.}
\label{fig:stl0}
\end{figure}

\begin{definition}
Given a program~$\aprogram$,
     sets $\asetofpropagators_\leftstate$ and ${\asetofpropagators_\rightstate}$ of p-conditions,
     a generating function $\themainprogram$,
and a witness function $\thedependingprogram$,
a {\em two-layer template graph}
$\simpletemplate{\asetofpropagators_\leftstate}
     {\asetofpropagators_\rightstate}
     {\themainprogram}{\thedependingprogram}
     (\aprogram)$
is a graph defined
as follows:
\begin{itemize}\itemsep0pt\parskip0pt\parsep0pt
 \item The set of nodes is, as in the previous two-layer graphs, the set of
       states relative to $\atoms{\themainprogram(\aprogram)}$ and
       $\atoms{\thedependingprogram,\aprogram,\atoms{\mainprogram{\aprogram}}}$; and
 \item The transition rules are the rules presented in Figure~\ref{fig:trfirst}
       except the rules $\unitleft$ and $\unitright$, that are replaced by the
       rules $\propagateleft$ and $\propagateright$ presented in
       Figure~\ref{fig:stl0}.
\end{itemize}
\end{definition}


\paragraph{Description of the Template.}
We call the state
$(\emptyset,\emptyset)_{\leftstate}$ {\em initial}.
Note how the rules $\propagateleft$ and $\propagateright$ in
$\simpletemplate{\asetofpropagators_\leftstate}
     {\asetofpropagators_\rightstate}
     {\themainprogram}{\thedependingprogram}
     (\aprogram)$
refer to the parameters ${\asetofpropagators_\leftstate}$,
${\asetofpropagators_\rightstate}$, ${\themainprogram}$ and ${\thedependingprogram}$
of the graph template.
Varying these parameters will allow us to specify transition systems that
capture different disjunctive answer set solvers.
Intuitively, the parameters ${\asetofpropagators_\leftstate}$ and
${\asetofpropagators_\rightstate}$ are sets of p-conditions defining a
propagation rule on generate and test side of computation, respectively.

The instantiation
$\simpletemplate{\uppropagators}{\uppropagators}{\thegencmodels}{\thetestcmodels}(\aprogram)$
of the two-layer graph template results in $\dpgraph{\aprogram}$.
Indeed, the graphs share the same nodes.
Also their rules $\failleft$, $\failright$, $\decideleft$, $\decideright$, $\backtrackleft$,
$\backtrackright$ and $\failcross$ coincide.
It is easy to see that a literal $\aliteral$ is in
$\outprop{\uppropagators}{\thegencmodels(\aprogram)}{\astringofliterals}$ if
and only if the transition rule $\unitleft$ in $\dpgraph{\aprogram}$ is
applicable in $(\astringofliterals,\emptyset)_\leftstate$
and supports the transition to a state
$(\astringofliterals\aliteral,\emptyset)_\leftstate$.
Thus, the transition rule $\propagateleft$ supports the transition from
$(\astringofliterals,\emptyset)_\leftstate$
to $(\astringofliterals\aliteral,\emptyset)_\leftstate$ if and only if the
transition rule $\unitleft$ supports the same transition.
A similar statement holds for the case of $\propagateright$ and $\unitright$.

Recall that in Section~\ref{sec:cmodelswithoutlearning} we showed that  {\cmodels} implementing
backtracking can be defined using the graph $\dpgraph{\aprogram}$.
The fact that instantiation $\simpletemplate{\uppropagators}{\uppropagators}{\thegencmodels}{\thetestcmodels}(\aprogram)$ coincides with $\dpgraph{\aprogram}$ illustrates that the introduced template is sufficient for capturing existing solvers. Next section demonstrates that the proposed template is  suitable for capturing  {\gnt} and {\dlv}.

\paragraph{Instantiation: Approximating and Ensuring Pairs.}
In the definition of the two-layer template graph
$\simpletemplate{\asetofpropagators_\leftstate}
     {\asetofpropagators_\rightstate}
     {\themainprogram}{\thedependingprogram}
     (\aprogram)$
     we pose no restrictions on its four key parameters:
     sets $\asetofpropagators_\leftstate$, ${\asetofpropagators_\rightstate}$ of p-conditions, and  generating and witness functions $\themainprogram$, $\thedependingprogram$.
In practice, when this template is utilized to model, characterize, and elicit disjunctive solvers these four parameters  exhibit specific properties. We now introduce terminology that allows us to specify essential properties of these parameters that will translate into correctness of  solvers captured by properly instantiated template. On the one hand, we introduce the conditions 
on  generating and witness functions under which we call these functions ''approximating`` and ''ensuring``, respectively. On the other hand, we couple these conditions with restrictions on sets of p-conditions so that we can speak of (i) approximating-pair  $(\asetofpropagators_\genleftgen,\genleftgen)$  for a set $\asetofpropagators_\genleftgen$ of p-conditions  and a generating function $\genleftgen$, and (ii)
ensuring-pair  $(\asetofpropagators_\genrighttest,\genrighttest)$ 
 for a set $\asetofpropagators_\genrighttest$ of p-conditions  and a witness function $\genrighttest$. For such pairs, the  template instantiation
$\simpletemplate{\asetofpropagators_\genleftgen}{\asetofpropagators_\genrighttest}
                {\genleftgen}{\genrighttest}(\aprogram)$ results in a graph that checks  
                stable models of $\aprogram$. 
As a result, when we characterize such solvers as {\gnt} and {\dlv} by means of the two-layer template we focus on (i)~specifying their generating and witness function as well as their sets of p-conditions, and (ii)~illustrating that they form proper approximating and ensuring pairs. This also brings us to the realization that an inception of a novel solver can be triggered by a creation of a novel approximation and ensuring pairs or their combinations. 
                We now make these ideas precise.

For types $\atypeofmodel$ and
$\atypeofmodel_1$, we say that a generating function
$\themainprogram$ is $\atypeofmodel_1$-approximating with respect to
type $\atypeofmodel$ if
for any program~$\aprogram$:
\begin{enumerate}\itemsep0pt\parskip0pt\parsep0pt
  \item For any stable model $\astringofliterals$ of $\aprogram$ there is
        a {\atypeofmodel$_1$}-model $\astringofliterals_1$ of
        $\themainprogram(\aprogram)$ such that
        $\astringofliterals=\restriction{\astringofliterals_1}{\atoms{\aprogram}}$; and
  \item For any {\atypeofmodel$_1$}-model $\asetofliterals$ of
        $\themainprogram(\aprogram)$, $\asetofliterals_{|\atoms{\aprogram}}$ is
        a $\atypeofmodel$-model of $\aprogram$.
\end{enumerate}

Consider the generating function $\thecnfcomp(\aprogram)$
that returns a CNF formula, which stands for
the completion $\completion{\aprogram}$ converted to CNF using straightforward
equivalent transformations.
In other words, $\thecnfcomp(\aprogram)$ consists of clauses of two kinds
\begin{enumerate}
  \item the rules $\ahead\aspimplication\abody$ of the program written as
        clauses $\ahead\vee\opposite{\abody}$, and
  \item formulas of $\thecnfcomp(\aprogram)$ from~\eqref{eq:comp-second}
        converted to CNF using the distributivity of disjunction over
        conjunction.\footnote{It is essential that repetitions are not removed
        in the process of clausification. For instance,
        $\thecnfcomp(a\aspimplication\ not\ a)=(a\vee a)\wedge (\neg a\vee \neg a).$}
\end{enumerate}
The function $\thecnfcomp$
is {\classic}-approximating with respect to {\supported}.
Indeed,
\begin{enumerate}\itemsep0pt\parskip0pt\parsep0pt
  \item any stable model of a program $\aprogram$ is also a {\classic}-model of
        $\thecnfcomp(\aprogram)$, and
  \item any {\classic}-model of $\thecnfcomp(\aprogram)$ is a {\supported}-model
        of $\aprogram$.
\end{enumerate}
Since any supported model is also a classical model,
the $\thecnfcomp$ function is also {\classic}-approximating with respect to {\classic}.
Note that when a generating function $\themainprogram$ is
$\atypeofmodel_1$-approximating with respect to $\atypeofmodel$, then
enumerating all $\atypeofmodel_1$-models of $\themainprogram(\aprogram)$ results
in enumerating some  $\atypeofmodel$-models of $\aprogram$ modulo a restriction
to $\atoms{\aprogram}$.

For types $\atypeofmodel$ and
$\atypeofmodel_1$, and a witness function
$\thedependingprogram$, we say that $\thedependingprogram$ is {\em
{\atypeofmodel$_1$}-ensuring} with respect to $\atypeofmodel$ when
for any set~$\asetofliterals$ of
literals covering $\aprogram$ such that $\asetofliterals_{|\atoms{\aprogram}}$
is ${\atypeofmodel}$-model of $\aprogram$, $\asetofliterals_{|\atoms{\aprogram}}$ is a
stable model of $\aprogram$ if and only if
$\thedependingprogram(\aprogram,\asetofliterals)$ results in a program that
has no {\atypeofmodel$_1$}-model.

For instance, the witness function~$\thetestcmodels$ is {\classic}-ensuring
with respect to $\classic$.
Since any {\supported}-model is also a {\classic}-model, 
the function $\thetestcmodels$ is also {\classic}-ensuring with respect to
{\supported}.
It is easy to see that when a witness function $\thedependingprogram$ is
$\atypeofmodel_1$-ensuring with respect to $\atypeofmodel$, then
given any $\atypeofmodel$-model~$\astringofliterals$ of a program $\aprogram$
we may use the function $\thedependingprogram$ to test that
$\astringofliterals$ is also a stable model of $\aprogram$. Indeed, an application of
$\thedependingprogram$ resulting in a program that has no
$\atypeofmodel_1$-models translates into the statement that $\astringofliterals$
is a stable model of $\aprogram$.

These newly defined concepts of approximating and
ensuring functions provide the following characterization for
the set of stable models of a program $\aprogram$.

\begin{proposition}\label{prop:approx-ensuring}
For any types $\atypeofmodel$, $\atypeofmodel_1$ and $\atypeofmodel_2$,
generating function $\themainprogram$
that is $\atypeofmodel_1$-approximating with respect to~$\atypeofmodel$, witness function
$\thedependingprogram$ that is $\atypeofmodel_2$-ensuring
with respect to $\atypeofmodel$, and program~$\aprogram$,
the set of all stable models of $\aprogram$ is
$$\{\restriction{\astringofliterals}{\atoms{\aprogram}}\mid\astringofliterals\textrm{ is a }\atypeofmodel_1\textrm{-model of }
\themainprogram(\aprogram)\textrm{ and }
\thedependingprogram(\aprogram,\astringofliterals)\textrm{ has no }
\atypeofmodel_2\textrm{-models}\}.$$
\end{proposition}

We now introduce the notion of
ensuring and approximating pairs that permit an operational use of generating and witness functions, by
matching them with a relevant set of propagators.
We call a pair $(\asetofpropagators,\themainprogram)$ of a set of p-conditions
and a generating function an {\em approximating-pair} with respect to
$\atypeofmodel$ if for some type $\atypeofmodel_1$, the set $\asetofpropagators$
is $\atypeofmodel_1$-enforcing and the function $\themainprogram$ is
$\atypeofmodel_1$-approximating with respect to $\atypeofmodel$.
For example, the pair $(\uppropagators,\thecnfcomp)$ is an approximating-pair
with respect to {\supported} as well as to {\classic}.
The $(\uppropagators,\thegencmodels)$ is also an approximating-pair
with respect to {\supported} as well as to {\classic}.


We call a pair $(\asetofpropagators,\thedependingprogram)$ of a set of
p-conditions and a witness function an {\em ensuring-pair} with respect to
$\atypeofmodel$ if for some type $\atypeofmodel_1$, the set $\asetofpropagators$
is $\atypeofmodel_1$-enforcing and the function $\thedependingprogram$ is
$\atypeofmodel_1$-ensuring with respect to $\atypeofmodel$.
For example, the pair $(\uppropagators,\thetestcmodels)$ is an ensuring-pair
with respect to any defined type.

We are now ready to state the main result of this section.

\begin{theorem}~\label{STS-correctness}
For any program $\aprogram$, any type $\atypeofmodel$,
any $(\asetofpropagators_\genleftgen,\genleftgen)$ approximating-pair with
respect to~$\atypeofmodel$, and
any $(\asetofpropagators_\genrighttest,\genrighttest)$ ensuring-pair with
respect to~$\atypeofmodel$,
the graph
$\simpletemplate{\asetofpropagators_\genleftgen}{\asetofpropagators_\genrighttest}
                {\genleftgen}{\genrighttest}(\aprogram)$
checks the stable models of $\aprogram$.
\end{theorem}

Theorem~\ref{STS-correctness} illustrates
how the template
$\simpletemplate{\asetofpropagators_\genleftgen}{\asetofpropagators_\genrighttest}
                {\genleftgen}{\genrighttest}(\aprogram)$
can serve as a framework for
defining transitions systems that result in correct algorithms for deciding
whether a program $\aprogram$ has a stable model.
The facts that $(\uppropagators,\thegencmodels)$ is an approximating-pair with
respect to $\classic$ and that $(\uppropagators,\thetestcmodels)$ is an
ensuring-pair with respect to $\classic$, together with
Theorem~\ref{STS-correctness}, subsume the result of
Proposition~\ref{correctnessCM}.

We now state propositions that capture interesting properties about states of
the graph
$\simpletemplate{\asetofpropagators_\genleftgen}{\asetofpropagators_\genrighttest}
                {\genleftgen}{\genrighttest}(\aprogram)$.
The former proposition concerns states with the label $\leftstate$, the latter
concerns states with the label $\rightstate$.

\begin{proposition}~\label{lemma:statement-e}
For any type {\atypeofmodel},
generating function $\genleftgen$,
witness function $\genrighttest$,
{\atypeofmodel}-enforcing set of p-conditions $\asetofpropagators_\genleftgen$,
set of p-conditions  $\asetofpropagators_\genrighttest$, and
program $\aprogram$, if no left-rule is applicable in some state
$(\aliteral_1.\cdots.\aliteral_{\athirdindex_1},
\anotherliteral_1.\cdots.\anotherliteral_{\athirdindex_2})_\leftstate$
in
$\simpletemplate{\asetofpropagators_\genleftgen}{\asetofpropagators_\genrighttest}
                {\genleftgen}{\genrighttest}(\aprogram)$
reachable from the initial state, then
$\aliteral_1.\cdots.\aliteral_{\athirdindex_1}$ is a {\atypeofmodel}-model
of $\genleftgen(\aprogram)$.
\end{proposition}

\begin{proposition}~\label{lemma:statement-d}
For any types {\atypeofmodel$_1$} and {\atypeofmodel$_2$},
generating function $\genleftgen$
witness function $\genrighttest$,
{\atypeofmodel$_1$}-enforcing set of p-conditions
$\asetofpropagators_\genleftgen$,
{\atypeofmodel$_2$}-enforcing set of p-conditions
$\asetofpropagators_\genrighttest$,
program $\aprogram$, and a state
$(\aliteral_1.\cdots.\aliteral_{\athirdindex_1},
\anotherliteral_1.\cdots.\anotherliteral_{\athirdindex_2})_\rightstate$ in
$\simpletemplate{\asetofpropagators_\genleftgen}{\asetofpropagators_\genrighttest}
                {\genleftgen}{\genrighttest}(\aprogram)$
reachable from the initial state, the following conditions hold:
\setlength\leftmargin{25pt}
\begin{enumerate}\itemsep0pt\parskip0pt\parsep0pt
 \item[(a)] $\genrighttest(\aprogram,\aliteral_1.\cdots.\aliteral_{\athirdindex_1})$ is defined,
 \item[(b)] $\anotherliteral_1.\cdots.\anotherliteral_{\athirdindex_2}$ is a set of literals over
            $\genrighttest(\aprogram,\astringofliterals)$,
 \item[(c)] $\aliteral_1.\cdots.\aliteral_{\athirdindex_1}$ is a {\atypeofmodel$_1$}-model
            of $\genleftgen(\aprogram)$, and
 \item[(d)] If no right-rule is applicable to
            $(\aliteral_1.\cdots.\aliteral_{\athirdindex_1},
              \anotherliteral_1.\cdots.\anotherliteral_{\athirdindex_2})_\rightstate$
            then $\anotherliteral_1.\cdots.\anotherliteral_{\athirdindex_2}$ is a
            {\atypeofmodel$_2$}-model of
            $\genrighttest(\aprogram,\aliteral_1.\cdots.\aliteral_{\athirdindex_1})$.
\end{enumerate}
\end{proposition}

\section{Applications of the Template}\label{sec:applications}
Section~\ref{sec:cmodelswithoutlearning}
 illustrates how  {\cmodels} implementing
backtracking can be defined using the graph $\dpgraph{\aprogram}$, while
the previous section states that the instantiation
$\simpletemplate{\uppropagators}{\uppropagators}{\thegencmodels}{\thetestcmodels}(\aprogram)$ 
of the two-layer graph template results in $\dpgraph{\aprogram}$. Thus, this template is suitable for capturing computations of {\cmodels}.
 In this
section, we show how the template also captures the solvers {\gnt} and
{\dlv} without backjumping.
Then, we discuss how the framework facilitates
the design of new abstract solvers and their comparison, by means of inspecting the structures
of the related graphs.

\paragraph{Abstract {\gnt}.}
We now show how the procedure underlying disjunctive solver~{\gnt} can be
captured by the two-layer template.
Unlike solver {\cmodels} that uses the {\dpll} procedure for generating and
testing, system {\gnt} uses the {\smodels} procedure for respective tasks.
Recall that the {\smodels}
procedure finds stable models for non-disjunctive logic programs, while the
{\dpll} procedure finds classical models.
The graph~{\sc sm}$_\aprogram$ (Section~\ref{sec:singlelayertemplate})
captures the computation underlying {\smodels} just as
the graph $\ordinarydp{\aprogram}$ captures the computation
underlying {\dpll}.
It forms a basis for
devising the transition system suitable to describe {\gnt}.
The graph describing the general
structure of {\gnt} is obtained from the graph template
$\simpletemplate{\smpropagators}{\smpropagators}{\themainprogram}{\thedependingprogram}(\aprogram)$
that rely on the set $\smpropagators$ of p-contitions.\footnote{The
graph template
$\simpletemplate{\smpropagators}{\smpropagators}{\themainprogram}{\thedependingprogram}(\aprogram)$ corresponds to the graph
$\smgraph{\themainprogram(\aprogram),\thedependingprogram}$ defined in \cite{blm14}.}

\citeN{jan06} define
the generating function
$\thegengnt$ and the witness function~$\thetestgnt$ used in \gnt.
We present these
definitions in~\eqref{eqn:gengnt} and~\eqref{eqn:testgnt}.\footnote{The
presented functions $\thegengnt$ and $\thetestgnt$ capture the essence of functions
  $Gen$ and $Test$
  defined by Janhunen et al., but they are not identical. Our language of
  disjunctive programs includes rules with empty heads.
  This allows us a more concise description.}
For a disjunctive program $\aprogram$,
by $\aprogram_N$ we denote the set of non-disjunctive rules of
$\aprogram$, by $\aprogram_D$ we denote the set of disjunctive rules $\aprogram\setminus\aprogram_N$.
For each atom $\anatom$ in $\atoms{\aprogram}$ let $\anatom^r$ and $\anatom^s$
be new atoms.

\begin{equation}\label{eqn:gengnt}
\begin{array}{rll}
\gengnt{\aprogram}=
  & \{\anatom\aspimplication\abody,\aspnot\anatom^r
                   \mid\ahead\logicalor\anatom\aspimplication\abody\in\aprogram_D\}\setunion&\\
  & \{\anatom^r\aspimplication\aspnot\anatom
                   \mid\ahead\logicalor\anatom\aspimplication\abody\in\aprogram_D\}\setunion&\\
  & \{\aspimplication\opposite{\ahead},\abody
                   \mid\ahead\aspimplication\abody\in\aprogram_D\}\setunion&\\
  & \aprogram_N\setunion \\
  & \{\anatom^s\aspimplication\opposite{\ahead\setminus\{\anatom\}},\abody
                   \mid\ahead\logicalor\anatom\aspimplication\abody\in\aprogram_D\}
                       \setunion&\\
  &           \{\aspimplication\anatom,\aspnot\anatom^s
                   \mid\anatom\vee\ahead\aspimplication\abody\in\aprogram_D\}&\\
\end{array}
\end{equation}

\begin{equation}\label{eqn:testgnt}
\begin{array}{rll}
\testgnt{\aprogram}{\asetofliterals}=
             & \{\anatom\aspimplication\abody,\aspnot\anatom^r \mid
                 \ahead\logicalor\anatom\aspimplication\abody\in\reduct{\aprogram}{\asetofliterals}_D,
                 \anatom\in\asetofliterals,
                 \abody\subseteq\asetofliterals\}
                 \setunion\\
             & \{\anatom^r\aspimplication\aspnot\anatom \mid
                 \ahead\logicalor\anatom\aspimplication\abody\in{\aprogram}\}
                 \setunion\\
             & \{\aspimplication\opposite{\ahead},\abody\mid
                 \ahead\aspimplication\abody\in\reduct{\aprogram}{\asetofliterals}_D,
                 \abody\subseteq\asetofliterals\}
                 \setunion\\
             & \{\anatom\aspimplication\abody\mid
                 \anatom\aspimplication\abody\in\reduct{\aprogram}{\asetofliterals}_N,
                 \anatom\in\asetofliterals,
                 \abody\subseteq\asetofliterals\}
                 \setunion\\
             & \{\aspimplication\restriction{\asetofliterals}{\atoms{\aprogram}}\}
\end{array}
\end{equation}

\begin{figure}
\footnotesize{ \input{the.example.GT.tex} }\normalsize
\caption{Example of path through the graph
$\smgraph{\{\anatom \aspimplication \athirdatom;
\anotheratom \aspimplication \athirdatom;
\athirdatom \aspimplication \anatom, \anotheratom;
\anatom \logicalor \anotheratom \aspimplication\}}$.}
\label{figure:example.GT}
\end{figure}

By $\smgraph{\aprogram}$ we denote the graph
$\simpletemplate{\smpropagators}{\smpropagators}{\thegengnt}{\thetestgnt}(\aprogram)$.
The graph $\smgraph{\aprogram}$ captures the {\gnt} procedure by \citeN{jan06}
in a similar way as the graph $\dpgraph{\aprogram}$ captures the {\cmodels}
procedure of {\sc dp-assat-proc} in Section~\ref{sec:cmodelswithoutlearning}.
Figure~\ref{figure:example.GT} presents an example of a path in a
graph $\smgraph{\{\anatom \aspimplication \athirdatom;
\anotheratom \aspimplication \athirdatom;
\athirdatom \aspimplication \anatom, \anotheratom;
\anatom \logicalor \anotheratom \aspimplication\}}$.
From the formal results by \citeN{jan06} it immediately follows that
$\thegengnt$ is {\stable}-approximating with respect to {\classic} and $\thetestgnt$ is
{\stable}-ensuring with respect to {\classic}.
The pair $(\smpropagators,\thegengnt)$ is an approximating-pair with
respect to {\classic}, while $(\smpropagators,\thetestgnt)$ is an ensuring-pair
with respect to {\classic}.
The following result immediately follows from
Theorem~\ref{STS-correctness}.\footnote{Corollary~\ref{GNT-correctness}
corresponds to Theorem~5 in~\cite{blm14}.}

\begin{corollary}~\label{GNT-correctness}
For any $\aprogram$ the graph
$\smgraph{\aprogram}$
checks the stable models of $\aprogram$.
\end{corollary}


\paragraph{Abstract {\dlv} without Backjumping.}
This section introduces graphs that capture the answer set solver {\dlv}
without backjumping.
The generate layer, i.e., the left-rule layer, is reminiscent to
the {\smodels} algorithm except it does not use $\theunfounded$.
The test layer applies the {\dpll} procedure to a witness formula.

The graph templates
$\simpletemplate{\smdisjpropagators}{\uppropagators}{\themainprogram}{\thedependingprogram}(\aprogram)$
describes the general
structure of {\dlv}.
The generating function~$\thegendlv$ is the identity function as in
\eqref{eqn:dlvgen}, and the witness function $\thetestdlv$ follows in
\eqref{eqn:dlvtest}.

\begin{equation}\label{eqn:dlvgen}
\begin{array}{lll}
\gendlv{\aprogram}=
             & \aprogram\\
\end{array}
\end{equation}

\begin{equation}\label{eqn:dlvtest}
\begin{array}{lll}
\testdlv{\aprogram}{\asetofliterals}=
             & \{\disjunctionof{(\abody\setintersection\answersetof{\asetofliterals})}
                 \logicalor\disjunctionof{\opposite{\ahead'}} \mid
                 \ahead\aspimplication\abody\in\reduct{\aprogram}{\answersetof{\amodel}},
                 \abody\subseteq\amodel,\ahead'=\ahead\cap\answersetof{\amodel}\}
               \setunion\\
             & 
             
             \{\disjunctionof{(
             \restriction{\asetofliterals}{\atoms{\aprogram}}
             )}\}\\
\end{array}
\end{equation}

\begin{figure}
\footnotesize{ \input{the.example.DL.tex} }\normalsize
\caption{Example of path through the graph
$\simpletemplate{\smdisjpropagators}{\uppropagators}{\thegendlv}{\thetestdlv}(
\{\anatom \aspimplication \athirdatom;
\anotheratom \aspimplication \athirdatom;
\athirdatom \aspimplication \anatom, \anotheratom;
\anatom \logicalor \anotheratom \aspimplication\})$.}
\label{figure:example.DL}
\end{figure}

Following the results from \citeN{faber2002enhancing} and
\citeN{Koch:1999:SMC:1624218.1624229}, the generating function $\thegendlv$ is
{\supported}-approximating with respect to {\classic} while the witness function
$\thetestdlv$ is {\classic}-ensuring with respect to {\classic}.
The pair $(\smdisjpropagators,\thegendlv)$ is an approximating-pair
with respect to {\classic}, while $(\uppropagators,\thetestdlv)$ is an
ensuring-pair with respect to {\classic}.
The result below immediately follows from
Theorem~\ref{STS-correctness}.\footnote{Corollary~\ref{DLV-correctness} corresponds to
Theorem~6 in~\cite{blm14}.}

\begin{corollary}~\label{DLV-correctness}
For any $\aprogram$ the graph
$\simpletemplate{\smdisjpropagators}{\uppropagators}{\thegendlv}{\thetestdlv}(\aprogram)$
checks the stable models of $\aprogram$.
\end{corollary}
This corollary is an alternative proof of correctness for the \dlv algorithm previously stated
by \citeN{faber2002enhancing} and \citeN{Koch:1999:SMC:1624218.1624229} in terms
of pseudo-code.
Figure~\ref{figure:example.DL} presents an example of a path through one of
the graph describing abstract {\dlv}.

\paragraph{Designing new Graphs and Comparing Graphs.}
The two-layer graph template can be conveniently used to define new abstract
solvers.
For instance, one may choose to combine
$(\uppropagators,\thegencmodels)$ with $(\smpropagators,\thetestgnt)$
to obtain a solver captured by the graph template
$\simpletemplate{\uppropagators}{\smpropagators}{\thegencmodels}{\thetestgnt}(\aprogram)$.
Theorem~\ref{STS-correctness} provides a proof of correctness for the
procedure summarized by this family of graphs.
More generally, to obtain a new solver
one can combine any approximating-pair on the left side
of the graphs with any ensuring-pair on the right side with respect to the same
type. For instance, for any pair $(\asetofpropagators,\genrighttest)$ that is ensuring with
respect to $\classic$, the family of graphs
$\simpletemplate{\uppropagators}{\asetofpropagators}{\thecnfcomp}{\genrighttest}(\aprogram)$
captures a correct procedure for a disjunctive answer set solver.

In the following, we illustrate how abstract solvers can serve also as a convenient tool
 for comparing search procedures from an abstract point of view, by means of
comparison to the related graphs. In this respect we now state the result that
illustrates  a strong relation between  {\cmodels} and  {\dlv}. Indeed, their generate layer:

\begin{theorem}~\label{BigComp}
For any $(\asetofpropagators,\genrighttest)$
ensuring-pair with respect to $\classic$, and any program $\aprogram$,
the graphs
$\simpletemplate{\uppropagators}{\asetofpropagators}{\thecnfcomp}{\genrighttest}(\aprogram)$ and
$\simpletemplate{\smdisjpropagators}{\asetofpropagators}{\identity}{\genrighttest}(\aprogram)$
are identical graphs.
\end{theorem}



\section{Proofs}\label{sec:proofs}
\subsection{Proof of Theorem~\ref{thm:typecomplete}}
We start by stating several lemmas that will be instrumental in constructing
arguments for Theorem~\ref{thm:typecomplete}. Recall that
$\uppropagators=\{\theunit\}$.
 
\begin{lemma}\label{lem:classicmodelup}
The set $\uppropagators$ is $\classic$-complete.
\end{lemma}

In other words, for any program
$\aprogram$ and any complete and consistent set 
$\asetofliterals$ of literals over $\atoms{\aprogram}$, the set
$\asetofliterals$ is a $\classic$-model of $\aprogram$ iff 
$\setofuniformpcondition{\theunit}{\aprogram}{\asetofliterals}=\emptyset$. 

\begin{proof}
Left-to-right: Let $\asetofliterals$ be a $\classic$-model of $\aprogram$.
Our proof is by contradiction. Assume that
$\setofuniformpcondition{\theunit}{\aprogram}{\asetofliterals}\neq\emptyset$.
Take any literal $\aliteral$ from this set.
Then, the literal~$\aliteral$ is such that it does not belong to~$\asetofliterals$.
Also, there is a rule in~$\aprogram$ that is equivalent to a clause
$\aclause\logicalor\aliteral$ so that
all the literals of~$\opposite{\aclause}$
occur in $\asetofliterals$.
Since $\asetofliterals$ is a $\classic$-model of~$\aprogram$, we conclude that
$\aliteral\in\asetofliterals$. We derive a contradiction.

Right-to-left: Let $\setofuniformpcondition{\theunit}{\aprogram}{\asetofliterals}=\emptyset$.
By contradiction. Assume that $\asetofliterals$ is not a $\classic$-model of~$\aprogram$.
Then there is a rule in $\aprogram$ that is equivalent to a clause
$\aclause\logicalor\aliteral$ so that all the literals of $\opposite{\aclause}$
as well as $\opposite{\aliteral}$ occur in~$\asetofliterals$ (indeed,
$\asetofliterals$ is a complete set of literals over $\atoms{\aprogram}$ that
does not satisfy some rule in $\aprogram$). Since $\asetofliterals$ is
consistent, $\aliteral\not\in\asetofliterals$. It follows that
$\aliteral\in\setofuniformpcondition{\theunit}{\aprogram}{\asetofliterals}$.
We derive a contradiction.
\end{proof}

\begin{lemma}\label{lem:supporting}
For any program
$\aprogram$, any atom $\anatom$, and any sets $\asetofliterals$ and
$\asetofliterals'$ of literals such that
$\asetofliterals\subseteq\asetofliterals'$, if a rule in~$\aprogram$ is not a
supporting rule for $\anatom$ with respect to $\asetofliterals$ then this rule
is also not a supporting rule for $\anatom$ with respect to~$\asetofliterals'$.
\end{lemma}
\begin{proof}
By contradiction. Assume that there is a rule $\ahead\vee\anatom
\aspimplication\abody$ in $\aprogram$ such that it is not a supporting rule for
$\anatom$ with respect to $\asetofliterals$ but it is a supporting rule for
$\anatom$ with respect to $\asetofliterals'$.
It follows that
$\asetofliterals\cap(\opposite{\abody}\cup{\ahead})\neq\emptyset$, while
$\asetofliterals'\cap(\opposite{\abody}\cup{\ahead})=\emptyset$.
This contradicts the fact that $\asetofliterals\subseteq\asetofliterals'$. 
\end{proof}

We now generalize Lemma~4 from~\citeN{lier08} to the case of disjunctive
programs.
\begin{lemma}\label{lem:unf}
For any unfounded set $U$ on a consistent set $\astringofliterals$ of literals w.r.t. a
program~$\aprogram$ and any consistent and complete set $\asetofliterals$ of literals over $\atoms{\aprogram}$, 
if $\astringofliterals\subseteq\asetofliterals$ and $\asetofliterals\cap
U\neq\emptyset$, then $\asetofliterals$ is not a stable model of $\aprogram$.
\end{lemma}
\begin{proof}
By contradiction.
Assume that $\asetofliterals$ is a stable model of $\aprogram$.
Then, $\asetofliterals$ is a classic model of $\aprogram$ also.
By Theorem~\ref{thm:leone}, $\asetofliterals$ is such that there is no
non-empty subset of $\answersetof{\asetofliterals}$ such that it is an
unfounded set on $\asetofliterals$ w.r.t. $\aprogram$. Since
$\asetofliterals\cap U\neq\emptyset$, it follows that
$\asetofliterals\cap U$ is not an unfounded set on $\asetofliterals$ w.r.t.
$\aprogram$. It follows that for some rule
$\anatom\vee\ahead\aspimplication\abody\in\aprogram$ such that
$\anatom\in\asetofliterals\cap U$ all of the following conditions hold
\setlength\leftmargin{25pt}
\begin{enumerate}\itemsep0pt\parskip0pt\parsep0pt
  \item $\asetofliterals\cap\opposite{\abody}=\emptyset$,
  \item $\asetofliterals\cap U\cap\abody=\emptyset$, and
  \item $(\ahead\setminus (\asetofliterals\cap U)) \setintersection \asetofliterals=\emptyset$.
\end{enumerate}
Since $\asetofliterals\cap\opposite{\abody}=\emptyset$ and
$\astringofliterals\subseteq\asetofliterals$ it follows that
$\astringofliterals\cap\opposite{\abody}=\emptyset$.
Since $\asetofliterals\cap\opposite{\abody}=\emptyset$
and the fact that $\asetofliterals$ is consistent and complete set of literals over $\atoms{\aprogram}$,
$\answersetof{\abody}\subseteq\asetofliterals$. Consequently
$U\cap\answersetof{\abody}=\asetofliterals\cap U\cap \answersetof{\abody}=\emptyset$.
Since
$\astringofliterals\subseteq\asetofliterals$ and $(\ahead\setminus (\asetofliterals\cap U)) \setintersection \asetofliterals=\emptyset$, it follows that 
$(\ahead\setminus U) \setintersection \astringofliterals =\emptyset$.
Consequently, 
the set $U$ is not an unfounded set on~$\astringofliterals$.
\end{proof}

We are now ready to introduce the proof of Theorem~\ref{thm:typecomplete}.
\begin{proof}[Proof of Theorem~\ref{thm:typecomplete}]
{\em Statement 1.} We have to show that the set $\uppropagators$ is $\classic$-enforcing.
Lemma~\ref{lem:classicmodelup} states that the set $\uppropagators$ is $\classic$-complete.
Thus, we only ought to illustrate that $\uppropagators$ is $\classic$-sound.
Let
$\aprogram$ be any program, 
$\asetofliterals$ be any set of literals,
$\asetofliterals'$ be any
$\classic$-model of $\aprogram$ such that
$\asetofliterals\subseteq\asetofliterals'$.
We have to show that
$\outprop{\uppropagators}{\aprogram}{\asetofliterals}\subseteq\asetofliterals'$.
Let
$\aliteral$ be any literal in
$\outprop{\uppropagators}{\aprogram}{\asetofliterals}$.
We now show that $\aliteral\in\asetofliterals'$.
The p-condition $\theunit$ is the only member of the set $\uppropagators$. Thus,
$\outprop{\uppropagators}{\aprogram}{\asetofliterals}=\setofuniformpcondition{\theunit}{\aprogram}{\asetofliterals}$.
It follows that $l\in \setofuniformpcondition{\theunit}{\aprogram}{\asetofliterals}$.
By the conditions of $\theunit$ definition, 
there is a rule in $\aprogram$ that is equivalent to a clause
$\aclause\logicalor\aliteral$ so that
all the literals of $\opposite{\aclause}$
occur in $\asetofliterals$.
Since $\asetofliterals\subseteq\asetofliterals'$, it follows that 
all the literals of $\opposite{\aclause}$
occur in $\asetofliterals'$.
From the fact that $\asetofliterals'$ is $\classic$-model of $\aprogram$ it follows that 
$\asetofliterals'\models \aclause\logicalor\aliteral$. Consequently, $\aliteral\in \asetofliterals'$.

\medskip
{\em Statement 2.} We have to show that the subsets of
$\smdisjpropagators$ containing $\{\theunit,\linebreak[1]\theallcancel\}$
are $\supported$-enforcing. We first illustrate this property for the set
$\{\theunit,\linebreak[1]\theallcancel\}$. We call this set $ua$.
We start by showing that the set $ua$ is $\supported$-sound.
Let $\aprogram$ be any program, 
$\asetofliterals$ be any set of literals,
$\asetofliterals'$ be any
$\supported$-model of $\aprogram$ such that
$\asetofliterals\subseteq\asetofliterals'$.
We have to illustrate that the set $\outprop{ua}{\aprogram}{\asetofliterals}$
is a subset of $\asetofliterals'$. 
Consider any literal 
$\aliteral$ in the set $\outprop{ua}{\aprogram}{\asetofliterals}$. We now show
that $\aliteral$ is also in $\asetofliterals'$.

Case 1. $\aliteral\in \outprop{\theunit}{\aprogram}{\asetofliterals}$. Since
$\asetofliterals'$ is a $\supported$-model, $\asetofliterals'$ is also a
$\classic$-model. The rest of the argument follows the lines of proof in
{\em Statement 1}, which illustrates that $\uppropagators$ is $\classic$-sound.

Case 2. $\aliteral\in \outprop{\theallcancel}{\aprogram}{\asetofliterals}$. 
$\aliteral$ has the form $\neg\anatom$.
By the conditions of $\theallcancel$ definition, it follows that
there is no rule in $\aprogram$ supporting $\anatom$ with respect to $\asetofliterals$.
By Lemma~\ref{lem:supporting}, we derive that there is no rule in $\aprogram$
supporting $\anatom$ with respect to $\asetofliterals'$.
From the fact that $\asetofliterals'$ is $\supported$-model of $\aprogram$ it
follows that $\neg \anatom\in \asetofliterals'$. (Indeed, $\anatom$ may not be
a member of $\asetofliterals'$, while $\asetofliterals'$ is a complete set of
literals over $\atoms{\aprogram}$.)

\medskip
Second, we show that the set $ua$ is $\supported$-complete. Let $\aprogram$ be any
program, $\asetofliterals$ be any complete and consistent set of literals over
$\atoms{\aprogram}$.
We now show that $\asetofliterals$ is $\supported$-model of $\aprogram$ iff
$\outprop{ua}{\aprogram}{\asetofliterals}=\emptyset$.

Left-to-right: Let $\asetofliterals$ be a $\supported$-model of $\aprogram$.
By contradiction. Assume that the set $\outprop{ua}{\aprogram}{\asetofliterals}$
is not empty.
Then there is a literal $\aliteral$ in this set.

Case 1. $\aliteral\in \setofuniformpcondition{\theunit}{\aprogram}{\asetofliterals}$.
Since $\asetofliterals$ is also $\classic$-model of $\aprogram$, by
Lemma~\ref{lem:classicmodelup} we derive a contradiction.

Case 2. $\aliteral\in \setofuniformpcondition{\theallcancel}{\aprogram}{\asetofliterals}$.
$\aliteral$ has the form $\neg \anatom$.
By the conditions of $\theallcancel$ definition, it follows that
(i)  literal $\neg \anatom$ is such that it does not belong to $\asetofliterals$, and
(ii) there is no supporting rule in~$\aprogram$ for $\anatom$ with respect to $\asetofliterals$. 
     Since $\asetofliterals$ is a $\supported$-model of $\aprogram$, we conclude
     from (ii) that $\neg\anatom\in\asetofliterals$.
     This contradicts~(i).

Right-to-left: Assume $\outprop{ua}{\aprogram}{\asetofliterals}=\emptyset$.
By contradiction. Assume that $\asetofliterals$ is not a $\supported$-model of~$\aprogram$.
Then either $\asetofliterals$ is not a $\classic$-model of $\aprogram$ or there is an atom 
$\anatom\in\answersetof{\asetofliterals}$ such that
there is no supporting rule in $\aprogram$ for $\anatom$ with respect to $\asetofliterals$.
In the former case, when $\asetofliterals$ is not a $\classic$-model of
$\aprogram$, by Lemma~\ref{lem:classicmodelup} we derive a contradiction.
In the latter case, it follows that
$\neg\anatom\in\setofuniformpcondition{\theallcancel}{\aprogram}{\asetofliterals}$
by the conditions of the $\theallcancel$ definition.
We derive a contradiction.

\medskip
We now show that the set $\smdisjpropagators$ is $\supported$-enforcing.
We start by claiming that the set $\smdisjpropagators$ is $\supported$-sound.
Let
$\aprogram$ be any program, 
$\asetofliterals$ be any set of literals,
$\asetofliterals'$ be any
$\supported$-model of $\aprogram$ such that
$\asetofliterals\subseteq\asetofliterals'$.
We have to illustrate that the set
$\outprop{\smdisjpropagators}{\aprogram}{\asetofliterals}$
is a subset of $\asetofliterals'$. Consider any literal $\aliteral$ in the set
$\outprop{\smdisjpropagators}{\aprogram}{\asetofliterals}$. 
We show that $\aliteral$ is also in $\asetofliterals'$.
Given a proof that $ua$ is $\supported$-sound, it is only left to be proved that
for any literal $\aliteral$ that is in
 $\outprop{\thebacktrue}{\aprogram}{\asetofliterals}$, it also holds that $\aliteral\in \asetofliterals'$.
Consider literal $\aliteral\in \outprop{\thebacktrue}{\aprogram}{\asetofliterals}$.
By the definition of $\thebacktrue$ it follows that
there is a rule
$r=\ahead\logicalor\anatom\aspimplication\abody$ in $\aprogram$
so that (i) $\anatom\in\asetofliterals$,
and (ii) either $\opposite{\aliteral}\in\ahead$ or $\aliteral\in\abody$ and,
(iii) no other rule in $\aprogram$ is supporting $\anatom$ with respect to $\asetofliterals$.
By Lemma~\ref{lem:supporting} and (iii), we derive that
every rule other than $r$ is such that it is not a supporting
rule for $\anatom$ with respect to $\asetofliterals'$.
By (i) and the fact that $\asetofliterals\subseteq\asetofliterals'$,
$\anatom\in\asetofliterals'$. Since $\asetofliterals'$ is a $\supported$-model of 
$\aprogram$, it follows that $\asetofliterals\cap(\opposite{\abody}\cup\ahead)=\emptyset$. 
By the fact that $\asetofliterals'$ is a consistent and complete set of literals
over $\atoms{\aprogram}$ we conclude that 
$\abody\cup\opposite{\ahead}\subseteq\asetofliterals$. By (ii),
$\aliteral\in \asetofliterals'$.

\medskip
{\em Statement 3.} 
We have to show that the subsets of $\smpropagators$ containing
$\{\theunit,\linebreak[1]\theunfounded\}$ are $\stable$-enforcing.
We only illustrate this for the set $\{\theunit,\linebreak[1]\theunfounded\}$. 
We call this set $uu$.
The proof for other sets (i) relies on the fact that any $\stable$-model is
also a $\classic$ and $\supported$-model and (ii) follows the ideas presented in the proof of Statement~2.

We start by showing that the set $uu$ is $\stable$-sound.
Let
$\aprogram$ be any program, 
$\asetofliterals$ be any set of literals,
$\asetofliterals'$ be any
$\stable$-model of $\aprogram$ such that
$\asetofliterals\subseteq\asetofliterals'$.
We have to illustrate that the set
$\outprop{uu}{\aprogram}{\asetofliterals}$
is a subset of 
$\asetofliterals'$. 
Consider any literal 
$\aliteral$ in the set $\outprop{uu}{\aprogram}{\asetofliterals}$. We now show
that $\aliteral$ is also in $\asetofliterals'$.

Case 1. $\aliteral\in \outprop{\theunit}{\aprogram}{\asetofliterals}$. Since
$\asetofliterals'$ is a $\stable$-model, $\asetofliterals'$ is also a
$\classic$-model. The rest of the argument follows the lines of proof in
{\em Statement 1}, which illustrates that $\uppropagators$ is $\classic$-sound.

Case 2. $\aliteral\in \outprop{\theunfounded}{\aprogram}{\asetofliterals}$. 
Literal $\aliteral$ has the form $\neg\anatom$.
By the conditions of $\theunfounded$ definition, it follows that
there is a set $\asetofatoms$ containing $\anatom$ such that $\asetofatoms$ is
unfounded with respect to $\aprogram$.
By Lemma~\ref{lem:unf} and the fact that 
$\asetofliterals'$ is $\stable$-model of $\aprogram$ it follows that $\neg
\anatom\in \asetofliterals'$. (Indeed, consider a simple argument by contradiction.)

\medskip
Second, we show that the set $uu$
is $\stable$-complete. Let
$\aprogram$ be any program, 
$\asetofliterals$ be any complete and consistent set of literals over $\atoms{\aprogram}$. 
We now show that $\asetofliterals$ is $\stable$-model of $\aprogram$ iff
$\outprop{uu}{\aprogram}{\asetofliterals}=\emptyset$.

Left-to-right: Let $\asetofliterals$ be a $\stable$-model of $\aprogram$.
By contradiction. Assume that the set $\outprop{uu}{\aprogram}{\asetofliterals}$
is not empty.
Then there is a literal $\aliteral$ in this set.

Case 1. $\aliteral\in \setofuniformpcondition{\theunit}{\aprogram}{\asetofliterals}$.
Since $\asetofliterals$ is also $\classic$-model of $\aprogram$, by
Lemma~\ref{lem:classicmodelup} we derive a contradiction.

Case 2. $\aliteral\in \setofuniformpcondition{\theunfounded}{\aprogram}{\asetofliterals}$.
Literal $\aliteral$ has the form $\neg \anatom$.
By the conditions of $\theunfounded$ definition, it follows that
(i)  literal $\neg \anatom$ is such that it does not belong to $\asetofliterals$, and
(ii) there is a set $\asetofatoms$ containing $\anatom$ such that
     $\asetofatoms$ is unfounded with respect to $\aprogram$.
Since $\asetofliterals$ is a $\stable$-model of $\aprogram$, we conclude
from (ii) that $\neg\anatom\in\asetofliterals$. This contradicts~(i).

Right-to-left: Assume that $\outprop{uu}{\aprogram}{\asetofliterals}=\emptyset$.
By contradiction. Assume that $\asetofliterals$ is not a $\stable$-model of~$\aprogram$.
By Theorem~\ref{thm:leone}, either $\asetofliterals$ is not a $\classic$-model
of $\aprogram$ or there is a non-empty subset of $\answersetof{\astringofliterals}$ that 
is an unfounded set on $\astringofliterals$ with respect to $\aprogram$.
In the former case, when $\asetofliterals$ is not a $\classic$-model of
$\aprogram$, by Lemma~\ref{lem:classicmodelup} we derive a contradiction.
In the latter case, it follows that there is some atom $\anatom$ in an existing unfounded set so that
$\neg\anatom\in\setofuniformpcondition{\theunfounded}{\aprogram}{\asetofliterals}$
by the conditions of the $\theunfounded$ definition.
We derive a contradiction.
\end{proof}

\subsection{Proofs of Theorems~\ref{thm:minitemplate},~\ref{STS-correctness},
Propositions~\ref{thm:firstgraph},~\ref{prop:approx-ensuring},~\ref{lemma:statement-e},~\ref{lemma:statement-d}
}
We start by the proof of Theorem~\ref{STS-correctness}.
We skip the proof of Theorem~\ref{thm:minitemplate} as 
it relies on the same proof techniques that proof of Theorem~\ref{STS-correctness} exhibits.
The proof of Theorem~\ref{STS-correctness}
relies on auxiliary lemmas as well as proofs of Propositions~\ref{prop:approx-ensuring},~\ref{lemma:statement-e},~\ref{lemma:statement-d}
that follow. We conclude this section with the proof of Proposition~\ref{thm:firstgraph}.


\begin{lemma}~\label{lemma:finite}
Let $\genleftgen$ 
be a generating function and $\genrighttest$ be a witness function.

Let $\asetofpropagators_\genleftgen$ and $\asetofpropagators_\genrighttest$
be sets of
p-conditions.

Then for any $\aprogram$, the graph
$\simpletemplate{\asetofpropagators_\genleftgen}{\asetofpropagators_\genrighttest}
                {\genleftgen}{\genrighttest}(\aprogram)$
is finite and acyclic.
\end{lemma}
\begin{proof}
Consider the states of the graph
$\simpletemplate{\asetofpropagators_\genleftgen}{\asetofpropagators_\genrighttest}
                {\genleftgen}{\genrighttest}(\aprogram)$.
The string $\astringofliterals$ of states of the form
$(\astringofliterals,
\anotherstringofliterals)_\aside$
or of the type $\terminalstate{\astringofliterals}$
is built over a set of atoms which is bounded by the size of $\aprogram$. Also,
$\astringofliterals$ does not allow repetitions. Thus, there is a finite
number of possible strings $\astringofliterals$ in the states
$(\astringofliterals,
\anotherstringofliterals)_\aside$
or $\terminalstate{\astringofliterals}$.
It immediately follows that there is a finite number of states $\terminalstate{\astringofliterals}$
in
$\simpletemplate{\asetofpropagators_\genleftgen}{\asetofpropagators_\genrighttest}
                {\genleftgen}{\genrighttest}(\aprogram)$.

Consider the right side of a state of the form
$(\astringofliterals,
\anotherstringofliterals)_\aside$.
Since $\genrighttest(\aprogram,\astringofliterals)$ has a finite number of
atoms and there is a finite number of possible $\astringofliterals$, 
the set of atoms over which $\anotherstringofliterals$
is built is finite.
Consequently, there is a finite number of possible $\anotherstringofliterals$.
We conclude that there is 
 a finite number of possible states
$(\astringofliterals,
\anotherstringofliterals)_\aside$.
Thus the set of states is finite in 
$\simpletemplate{\asetofpropagators_\genleftgen}{\asetofpropagators_\genrighttest}
{\genleftgen}{\genrighttest}(\aprogram)$.

For any string $\astringofliterals$ of literals, by
$\length{\astringofliterals}$ we denote the length of this string.
Any string $\astringofliterals$ of literals can be written
$\astringofliterals_0\decision{\aliteral_1}\astringofliterals_1
\ldots\decision{\aliteral_\athirdindex}\astringofliterals_\athirdindex$,
where $(\decision{\aliteral_\anindex})_{1\leq\anindex\leq\athirdindex}$
contains all the decision literals of $\astringofliterals$.
Let us call $\depth{\astringofliterals}$ the sequence
$\length{\astringofliterals_0},\length{\astringofliterals_1}
\ldots\length{\astringofliterals_\athirdindex}$.
We then write $\astringofliterals\leq\astringofliterals'$ iff
$\depth{\astringofliterals}\lexlower\depth{\astringofliterals'}$
where $\lexlower$ is the lexicographic order. Since the length of the sequence
$\depth{\astringofliterals}$ is bounded by the finite number of possible
decision literals, this is a well-founded order.
Finally, we say that
$(\astringofliterals,
\anotherstringofliterals)_\aside\leq
(\astringofliterals',
\anotherstringofliterals')_{\aside'}$
iff
$(\astringofliterals,\anotherstringofliterals,\aside)\lexlower
(\astringofliterals',\anotherstringofliterals',\aside')$
where $\lexlower$ is the lexicographic order and $\leftstate<\rightstate$.
This is clearly well-founded as it is the lexicographic composition of
well-founded orders.

If there is a transition from
$(\astringofliterals,
\anotherstringofliterals)_\aside$ to
$(\astringofliterals',
\anotherstringofliterals')_{\aside'}$ then
$(\astringofliterals,
\anotherstringofliterals)_\aside\leq
(\astringofliterals',
\anotherstringofliterals')_{\aside'}$ and
$(\astringofliterals,
\anotherstringofliterals)_\aside\neq
(\astringofliterals',
\anotherstringofliterals')_{\aside'}$.
This can be checked
simply for each of the rules.
Since the order is well-founded, there is no infinite path in the
graph. Consequently, the graph is acyclic.
\end{proof}

\medskip
\begin{proof}[Proof of Proposition~\ref{lemma:statement-e}]
We first show that
$\aliteral_1.\cdots.\aliteral_{\athirdindex_1}$ is
consistent.
By contradiction.
Assume that $\aliteral_1.\cdots.\aliteral_{\athirdindex_1}$ is inconsistent.
Then since $\failleft$
is not applicable
$\aliteral_1.\cdots.\aliteral_{\athirdindex_1}$ contains at
least one decision literal.
We now define $\anindex$ as
$\aliteral_1.\cdots.\aliteral_{\athirdindex_1}=$
${\aliteral_1}.\cdots.{\aliteral_{\anindex-1}}.\decision{\aliteral_\anindex}.
{\aliteral_{\anindex+1}}.\cdots.{\aliteral_{\athirdindex_1}}$
where $\decision{\aliteral_\anindex}$ is the rightmost
decision literal.
Since $\backtrackleft$ is not applicable
$\aliteral_1.\cdots.\aliteral_{\athirdindex_1}$ contains
no decision literal. We derive a contradiction.

Since $\decideleft$ is not applicable and
$\aliteral_1.\cdots.\aliteral_{\athirdindex_1}$
is consistent,
$\aliteral_1.\cdots.\aliteral_{\athirdindex_1}$
assigns all the atoms of
$\atoms{\genleftgen(\aprogram)}$. As a consequence
$\aliteral_1.\cdots.\aliteral_{\athirdindex_1}$ is a consistent and complete
set of literals that covers
$\atoms{\genleftgen(\aprogram)}$.
Finally, $\propagateleft$ is not applicable. So
$\asetofpropagators_\genleftgen(\aprogram,{\aliteral_1}.\cdots.{\aliteral_{\athirdindex_1}})$
is the empty set. Since $\asetofpropagators_\genleftgen$ is
{\atypeofmodel}-enforcing and hence {\atypeofmodel}-complete,
${\aliteral_1}.\cdots.{\aliteral_{\athirdindex_1}}$ is a
{\atypeofmodel}-model of $\genleftgen(\aprogram)$.
\end{proof}

\medskip
\begin{proof}[Proof of Proposition~\ref{lemma:statement-d}]
{\em Statements $(a-c)$}
We prove these statements by induction on the length of a path in
the graph
$\simpletemplate{\asetofpropagators_\genleftgen}{\asetofpropagators_\genrighttest}
                {\genleftgen}{\genrighttest}(\aprogram)$ from the initial state.
Since the statements trivially hold in the initial state of the graph, we only
have to prove that all transition rules of
$\simpletemplate{\asetofpropagators_\genleftgen}{\asetofpropagators_\genrighttest}
                {\genleftgen}{\genrighttest}(\aprogram)$
preserve the properties.

Statement $(c)$ trivially holds for all transitions but Crossing-rules
$\rightstate\leftstate$.
Statements $(a)$ and $(b)$ trivially hold for transitions due to Left rules,
Crossing-rules $\rightstate\leftstate$, $\failright$, $\failcross$.

Consider an edge due to one of the Right rules or Crossing-rules $\leftstate\rightstate$
from state $S=(\aliteral_1^0.\cdots.\aliteral_{\athirdindex_1^0}^0,
\anotherliteral_1^0.\cdots.\anotherliteral_{\athirdindex_2^0}^0)_{\aside^0}$
to state $S'=(\aliteral_1.\cdots.\aliteral_{\athirdindex_1},
\anotherliteral_1.\cdots.\anotherliteral_{\athirdindex_2})_\rightstate$
so that the statements $(a)$ and $(b)$ hold on $S$ (an inductive hypothesis). 
For the Right rules (excluding $\failright$), the left side of the state
remains unchanged. Thus, by induction hypothesis $(a)$ immediately follows.
Similarly, it is easy to see from the conditions of these rules that they also
preserve property $(b)$.
We now illustrate that 
the $\crossrule$ preserves $(a-c)$.

\medskip\emph{Case $\crossrule$:}
It follows that
(i)   $\aside_0=\leftstate$,
(ii)  $\aliteral_1.\cdots.\aliteral_{\athirdindex_1}=
       \aliteral_1^0.\cdots.\aliteral_{\athirdindex_1^0}^0$,
(iii) no left rule applies to $S$,
(iv)  $\anotherliteral_1.\cdots.\anotherliteral_{\athirdindex_1}=
       \anotherliteral_1^0.\cdots.\anotherliteral_{\athirdindex_1^0}^0=\emptyset$.
By Proposition~\ref{lemma:statement-e}, (i), and (iii) we conclude that
$\aliteral_1^0.\cdots.\aliteral_{\athirdindex_1^0}^0$ is a
{\atypeofmodel$_1$}-model of $\genleftgen(\aprogram)$.
By (ii), it follows that $\aliteral_1.\cdots.\aliteral_{\athirdindex_1}$ is
also a {\atypeofmodel$_1$}-model of $\genleftgen(\aprogram)$. Thus, $(c)$
holds. From the definition of $\genleftgen(\aprogram)$ it follows that the set
$\{\aliteral_1.\cdots.\aliteral_{\athirdindex_1}\}$ of literals covers $\Pi$.
It follows that 
$\genrighttest(\aprogram,\aliteral_1.\cdots.\aliteral_{\athirdindex_1})$ is defined. Thus $(a)$ holds.
From (iv), $(b)$ trivially follows as the right side of the state is empty.

\medskip
{\em Statement $(d)$}
We first show that
$\anotherliteral_1.\cdots.\anotherliteral_{\athirdindex_2}$ is
consistent.
By contradiction.
Assume that $\anotherliteral_1.\cdots.\anotherliteral_{\athirdindex_2}$ is inconsistent.
Then since $\failright$
is not applicable,
$\anotherliteral_1.\cdots.\anotherliteral_{\athirdindex_2}$ contains at
least one decision literal.
We now define $\anindex$ as
$\anotherliteral_1.\cdots.\anotherliteral_{\athirdindex_2}=$
${\anotherliteral_1}.\cdots.{\anotherliteral_{\anindex-1}}.\decision{\anotherliteral_\anindex}.
{\anotherliteral_{\anindex+1}}.\cdots.{\anotherliteral_{\athirdindex_2}}$
where $\decision{\anotherliteral_\anindex}$ is the rightmost
decision literal.
Since the rule $\backtrackright$ is not applicable
$\anotherliteral_1.\cdots.\anotherliteral_{\athirdindex_2}$ contains
no decision literal. We derive a contradiction.

Since $\decideright$ is not applicable and
$\anotherliteral_1.\cdots.\anotherliteral_{\athirdindex_2}$
is consistent, by $(b)$
$\anotherliteral_1.\cdots.\anotherliteral_{\athirdindex_2}$
assigns all the atoms of
$\atoms{\genrighttest(\aprogram,\aliteral_1.\cdots.\aliteral_{\athirdindex_1})}$. Thus,
$\anotherliteral_1.\cdots.\anotherliteral_{\athirdindex_2}$ is a consistent and complete
set of literals over $\atoms{\genrighttest(\aprogram,\aliteral_1.\cdots.\aliteral_{\athirdindex_1})}$.
Finally, $\propagateright$ is not applicable. So
$\asetofpropagators_\genrighttest(\genrighttest(\aprogram,\aliteral_1.\cdots.\aliteral_{\athirdindex_1}),
{\anotherliteral_1}.\cdots.{\anotherliteral_{\athirdindex_2}})$
is the empty set. Since $\asetofpropagators_\genrighttest$ is
{\atypeofmodel$_2$}-enforcing and hence {\atypeofmodel$_2$}-complete,
${\anotherliteral_1}.\cdots.{\anotherliteral_{\athirdindex_2}}$ is a
{\atypeofmodel$_2$}-model of $\genrighttest(\aprogram,\aliteral_1.\cdots.\aliteral_{\athirdindex_1})$.
\end{proof}

\begin{lemma}~\label{lemma:general}
Let {\atypeofmodel$_1$} and {\atypeofmodel$_2$} be some types in $\{\classic,\supported,\stable\}$.

Let $\genleftgen$ 
be a generating function and $\genrighttest$ be a witness function.

Let $\asetofpropagators_\genleftgen$ be a {\atypeofmodel$_1$}-enforcing set of
p-conditions and
 $\asetofpropagators_\genrighttest$ be a {\atypeofmodel$_2$}-enforcing set of
p-conditions. 

Let $\aprogram$ be a program.

Let
$(\aliteral_1.\cdots.\aliteral_{\athirdindex_1},
\anotherliteral_1.\cdots.\anotherliteral_{\athirdindex_2})_\aside$
be a state of
$\simpletemplate{\asetofpropagators_\genleftgen}{\asetofpropagators_\genrighttest}
                {\genleftgen}{\genrighttest}(\aprogram)$
reachable from the initial state.

Then:
\setlength\leftmargin{25pt}
\begin{enumerate}\itemsep0pt\parskip0pt\parsep0pt
  \item[(a)] 
        any {\atypeofmodel$_2$}-model of
        $\genrighttest(\aprogram,\aliteral_1.\cdots.\aliteral_{\athirdindex_1})$
        satisfies $\anotherliteral_\anindex$ if it satisfies
        all decision literals $\decision{(\anotherliteral_{\anotherindex})}$
        with $\anotherindex\leq\anindex$.
  \item[(b)] Any {\atypeofmodel$_1$}-model $\astringofliterals$ of $\genleftgen(\aprogram)$
        such that $\genrighttest(\aprogram,\astringofliterals)$ has no
        {\atypeofmodel$_2$}-model satisfies $\aliteral_\anindex$ if it satisfies
        all decision literals $\decision{\aliteral_\anotherindex}$ with
        $\anotherindex\leq\anindex$.
\end{enumerate}
\end{lemma}

\begin{proof}

We prove statements $(a)$ and $(b)$ by induction on the length of a path in
the graph
$\simpletemplate{\asetofpropagators_\genleftgen}{\asetofpropagators_\genrighttest}
                {\genleftgen}{\genrighttest}(\aprogram)$
from the initial state.
Since the statements trivially hold in the initial state of the graph, we only
have to prove that all transition rules of
$\simpletemplate{\asetofpropagators_\genleftgen}{\asetofpropagators_\genrighttest}
                {\genleftgen}{\genrighttest}(\aprogram)$
preserve the properties.

Consider an edge from the state 
$S=(\aliteral_1^0.\cdots.\aliteral_{\athirdindex_1^0}^0, \anotherliteral_1^0.\cdots.\anotherliteral_{\athirdindex_2^0}^0)_{\aside^0}$
to the state $S'=(\aliteral_1.\cdots.\aliteral_{\athirdindex_1},
\anotherliteral_1.\cdots.\anotherliteral_{\athirdindex_2})_\aside$
so that the statements $(a)$ and $(b)$ hold on $S$ (an inductive hypothesis). 

The statements $(a)$ and $(b)$ trivially hold for the case of transitions due to
$\failleft$, $\failright$, $\failcross$.

For the case of transition rules $\crossrule$,
$\backtrackright$, $\decideright$, $\propagateright$
it holds that
$\aliteral_1.\cdots.\aliteral_{\athirdindex_1} =
\aliteral_1^0.\cdots.\aliteral_{\athirdindex_1^0}^0$.
So by the induction hypothesis, $(b)$ trivially holds on
$(\aliteral_1.\cdots.\aliteral_{\athirdindex_1},
\anotherliteral_1.\cdots.\anotherliteral_{\athirdindex_2})_\aside$.
For these rules, we are left to show that $(a)$ holds on
$(\aliteral_1.\cdots.\aliteral_{\athirdindex_1},
\anotherliteral_1.\cdots.\anotherliteral_{\athirdindex_2})_\aside$.
Note that for the case of 
$\backtrackright$, $\decideright$, $\propagateright$, by
Proposition~\ref{lemma:statement-d} (a) it follows that
$\genrighttest(\aprogram,\aliteral_1.\cdots.\aliteral_{\athirdindex_1})$ is defined.

\medskip\emph{Case $\crossrule$:}
It follows that $\anotherliteral_1.\cdots.\anotherliteral_{\athirdindex_1}=
\anotherliteral_1^0.\cdots.\anotherliteral_{\athirdindex_1^0}^0=\emptyset$. 
Consequently, $(a)$ holds as right side of the state is empty.

\medskip\emph{Case $\backtrackright$.}
In this case, there is an index $\anindex$ such that
$\anotherliteral_1^0.\cdots.\anotherliteral_{\athirdindex_2^0}^0=
\anotherliteral_1^0.\cdots.\anotherliteral_{\anindex-1}^0.\decision{(\anotherliteral_\anindex^0)}.
\anotherliteral_{\anindex+1}^0.\cdots.\anotherliteral_{\athirdindex_2^0}^0$
and $\anotherliteral_1.\cdots.\anotherliteral_{\athirdindex_2-1}=
\anotherliteral_1^0.\cdots.\anotherliteral_{\anindex-1}^0$.
Also, by the conditions of $\backtrackright$, the string of literals
$\anotherliteral_1^0.\cdots.\anotherliteral_{\athirdindex_2^0}^0$ is
inconsistent.
Let $\asetofliterals$ be a {\atypeofmodel$_2$}-model of
$\genrighttest(\aprogram,\aliteral_1.\cdots.\aliteral_{\athirdindex_1})$.
Let $\anotherliteral_\anotherindex$ be a literal of
$\anotherliteral_1.\cdots.\anotherliteral_{\athirdindex_1}$.
Assume $\asetofliterals$
satisfies all decision literals $\decision{(\anotherliteral_{\anotherindex'})}$
with $\anotherindex'\leq\anotherindex$.
By the induction hypothesis, if $\anotherindex\neq\athirdindex_2$ then
$\asetofliterals$ satisfies $\anotherliteral_\anotherindex$.
It remains to prove that this is also true when
$\anotherindex = \athirdindex_2$.
Assume $\asetofliterals$ satisfies all the decision literals of
$\anotherliteral_1.\cdots.\anotherliteral_{\athirdindex_2}$. They include all
the decision literals of
$\anotherliteral_1.\cdots.\anotherliteral_{\athirdindex_2-1}$.
Then $\asetofliterals$ satisfies all the literals of
$\anotherliteral_1.\cdots.\anotherliteral_{\athirdindex_2-1}$ by the induction
hypothesis.
We now show that $\asetofliterals$ also satisfies
$\anotherliteral_{\athirdindex_2}$.

None of the literals
$\anotherliteral_{\anindex+1}^0\cdots\anotherliteral_{\athirdindex_2^0}^0$
is a decision literal.
Additionally,
$\anotherliteral_1.\cdots.\anotherliteral_{\athirdindex_2}=
\anotherliteral_1^0.\cdots.\anotherliteral_{\anindex-1}^0.\opposite{\anotherliteral_\anindex^0}$.
Since $\asetofliterals$
satisfies all the literals of
$\anotherliteral_1.\cdots.\anotherliteral_{\athirdindex_2-1}$, it satisfies all
the literals of $\anotherliteral_1^0.\cdots.\anotherliteral_{\anindex-1}^0$.
Since $\anotherliteral_1^0.\cdots.\anotherliteral_{\athirdindex_2^0}^0$ is
inconsistent,~$\asetofliterals$ cannot satisfy all of its literals, so
$\asetofliterals$ does not satisfy at least one literal of
$\anotherliteral_{\anindex}^0\cdots\anotherliteral_{\athirdindex_2^0}^0$.
By the contraposition of the induction hypothesis $(a)$, and since none of the
literals $\anotherliteral_{\anindex+1}^0\cdots\anotherliteral_{\athirdindex_2^0}^0$
is a decision literal, one of the literals not satisfied by
$\asetofliterals$ has to be $\anotherliteral_\anindex^0$.
So $\asetofliterals$ must
satisfy $\opposite{\anotherliteral_\anindex^0}$, that is
$\anotherliteral_{\athirdindex_2}$.

\medskip\emph{Case $\decideright$.} Obvious.

\medskip\emph{Case $\propagateright$.}
Let $\asetofliterals$ be a
{\atypeofmodel$_2$}-model of
$\genrighttest(\aprogram,\aliteral_1.\cdots.\aliteral_{\athirdindex_1})$.
Assume $\asetofliterals$ satisfies all the
decision literals of $\anotherliteral_1.\cdots.\anotherliteral_{\athirdindex_2}$.
Since for any propagator condition $\anotherliteral_{\athirdindex_2}$ is not a
decision literal, they are the decision literals of
$\anotherliteral_1.\cdots.\anotherliteral_{\athirdindex_2-1}$.
So $\asetofliterals$ satisfies all the literals of
$\anotherliteral_1.\cdots.\anotherliteral_{\athirdindex_2-1}$ by the induction hypothesis.
In other words
$\{\anotherliteral_1.\cdots.\anotherliteral_{\athirdindex_2-1}\}\subseteq\asetofliterals$.
Proving that $\asetofliterals$ satisfies
$\anotherliteral_{\athirdindex_2}$ will complete the proof.
We are given that
$\asetofpropagators_\genrighttest$ is {\atypeofmodel$_2$}-enforcing and hence
{\atypeofmodel$_2$}-sound. By definition of {\atypeofmodel$_2$}-soundness and the fact that 
$\{\anotherliteral_1.\cdots.\anotherliteral_{\athirdindex_2-1}\}\subseteq\asetofliterals$,
it follows that 
$\asetofpropagators_\genrighttest
        (\aprogram,
        \{\anotherliteral_1.\cdots.\anotherliteral_{\athirdindex_2-1}\})
\subseteq\asetofliterals$.
Since $\anotherliteral_{\athirdindex_2}\in\asetofpropagators_\genrighttest
        (\aprogram,
        \{\anotherliteral_1.\cdots.\anotherliteral_{\athirdindex_2-1}\})$,
also $\anotherliteral_{\athirdindex_2}\in\asetofliterals$.
In other words, $\asetofliterals$ satisfies
$\anotherliteral_{\athirdindex_2}$.

We are left to illustrate that transition rules $\backtrackcross$, $\backtrackleft$, $\decideleft$, $\propagateleft$
preserve properties $(a)$ and $(b)$. 
Since all of these rules are such that the right side of the
resulting state is $\emptyset$, clearly
 $(a)$ is preserved. 
 We will only illustrate that $\backtrackcross$ preserves $(b)$
 as the remaining cases for $(b)$ are similar to the 
arguments constructed above for the respective right rules and property $(a)$.

\medskip\emph{Case $\backtrackcross$.}
There is an index $\anindex$ such that
$\aliteral_1.\cdots.\aliteral_{\athirdindex_1-1}=
\anotherliteral_1^0.\cdots.\anotherliteral_{\anindex-1}^0$
and $\aliteral_1^0.\cdots.\aliteral_{\athirdindex_1^0}^0=
\aliteral_1^0.\cdots.\aliteral_{\anindex-1}^0.\decision{(\aliteral_\anindex^0)}.
\aliteral_{\anindex+1}^0.\cdots.\aliteral_{\athirdindex_1^0}^0$.
Let $\asetofliterals$ be a {\atypeofmodel$_1$}-model of $\genleftgen(\aprogram)$ such that
$\genrighttest(\aprogram,\asetofliterals)$ has no {\atypeofmodel$_2$}-model.
Assume that $\asetofliterals$ satisfies all the
decision literals of $\aliteral_1.\cdots.\aliteral_{\athirdindex_1}$.
Since $\aliteral_{\athirdindex_1}$ is not a
decision literal, they are the decision literals of
$\aliteral_1.\cdots.\aliteral_{\athirdindex_1-1}$.
So~$\asetofliterals$ satisfies all the literals of
$\aliteral_1.\cdots.\aliteral_{\athirdindex_1-1}$ by the induction
hypothesis.
Showing that $\asetofliterals$ satisfies
$\aliteral_{\athirdindex_1}$ will complete the proof.

Since the transition is justified by $\backtrackcross$, by Proposition~\ref{lemma:statement-d} (c),
$\aliteral_1^0.\cdots.\aliteral_{\athirdindex_1^0}^0$ is a
{\atypeofmodel$_1$}-model of $\genleftgen(\aprogram)$.
By Proposition~\ref{lemma:statement-d} (d)  and the fact that no right-rule applies,
$\anotherliteral_1^0.\cdots.\anotherliteral_{\athirdindex_2^0}^0$
is a {\atypeofmodel$_2$}-model of
$\genrighttest(\aprogram,\aliteral_1.\cdots.\aliteral_{\athirdindex_1})$.
So
$\genrighttest(\aprogram,\aliteral_1.\cdots.\aliteral_{\athirdindex_1})$
has a {\atypeofmodel$_2$}-model, hence $\asetofliterals$ does not satisfy all the literals
of $\aliteral_1^0.\cdots.\aliteral_{\athirdindex_1^0}^0$. Consequently,
since $\asetofliterals$ satisfies all the literals of
$\aliteral_1.\cdots.\aliteral_{\athirdindex_1-1}$, at least one literal
from $\decision{({\aliteral_\anindex}^0)}.
\aliteral_{\anindex+1}^0.\cdots.\aliteral_{\athirdindex_1^0}^0$
is not satisfied by $\asetofliterals$, which by the
contraposition of the induction hypothesis $(b)$ proves that
$\aliteral_\anindex^0 = \opposite{\aliteral_{\athirdindex_1}}$ is not satisfied by
$\asetofliterals$. This means that $\asetofliterals$ satisfies
$\aliteral_{\athirdindex_1}$.
\end{proof}

\begin{lemma}~\label{lemma:structure}
Let {\atypeofmodel$_1$} and {\atypeofmodel$_2$} be some types in $\{\classic,\supported,\stable\}$.

Let $\asetofpropagators_\genleftgen$ be a {\atypeofmodel$_1$}-enforcing set of
p-conditions. Let $\genleftgen$ be a generating function.

Let $\asetofpropagators_\genrighttest$ be a {\atypeofmodel$_2$}-enforcing set of
p-conditions. Let $\genrighttest$ be a witness function.

Let $\aprogram$ be a program. Then:
\setlength\leftmargin{25pt}
\begin{enumerate}\itemsep0pt\parskip0pt\parsep0pt
  \item any terminal state of
        $\simpletemplate{\asetofpropagators_\genleftgen}{\asetofpropagators_\genrighttest}
        {\genleftgen}{\genrighttest}(\aprogram)$
        reachable from the initial state and other than $\failstate$ is
        $\terminalstate{\astringofliterals}$, with
        $\seenasamodel{\astringofliterals}$ being a {\atypeofmodel$_1$}-model of
        $\genleftgen(\aprogram)$ such that
        $\genrighttest(\aprogram,\astringofliterals)$ has no {\atypeofmodel$_2$}-model,
  \item $\failstate$ is reachable from the initial state iff
        $\genleftgen(\aprogram)$ has no {\atypeofmodel$_1$}-model $\astringofliterals$ such that
        $\genrighttest(\aprogram,\astringofliterals)$ has no {\atypeofmodel$_2$}-model.
\end{enumerate}
\end{lemma}
\begin{proof}
We first illustrate that any terminal state is either
$\failstate$ or of the form $\terminalstate{\astringofliterals}$ for some
$\astringofliterals$.
By contradiction. Assume there is a terminal state of the form
$(\astringofliterals, \anotherstringofliterals)_\aside$.
Case 1. $\aside=\leftstate$. Then either a left rule or $\crossrule$ applies, so
$(\astringofliterals, \anotherstringofliterals)_\aside$ is not terminal. We
derive a contradiction.
Case 2. $\aside=\rightstate$.
Since $\failcross$ does not apply while no right-rule applies and no left-rule
applies, $\astringofliterals$ contains at least one decision literal.
Since $\backtrackcross$ is not applicable,
$\astringofliterals$ contains
no decision literal. We derive a contradiction.

\medskip
{\em Statement 1.}
Let $\terminalstate{\astringofliterals}$ be a terminal state
reachable from the
initial state. As it is different from the initial state there is a
transition leading to it. This transition can only be $\failright$. Let us call
$(\astringofliterals,
 \anotherstringofliterals)_\aside$ a state from which a transition
$\failright$ leads to $\terminalstate{\astringofliterals}$.
By the definition of $\failright$, we know that:
$\aside = \rightstate$, that $\anotherstringofliterals$ is inconsistent and
that $\anotherstringofliterals$ contains no decision literal.
By Lemma~\ref{lemma:general} item~$(c)$, the consistent set of literals obtained
from $\astringofliterals$ is a {\atypeofmodel$_1$}-model of $\genleftgen(\aprogram)$.

By Lemma~\ref{lemma:general} item~$(a)$, and as $\anotherstringofliterals$
contains no decision literal, any {\atypeofmodel$_2$}-model of
$\genrighttest(\aprogram,\astringofliterals)$ satisfies all the literals of 
$\anotherstringofliterals$.
Since $\anotherstringofliterals$ is inconsistent, any {\atypeofmodel$_2$}-model of
$\genrighttest(\aprogram,\astringofliterals)$ is inconsistent.
So $\genrighttest(\aprogram,\astringofliterals)$ has no {\atypeofmodel$_2$}-model.

We have just proved that $\astringofliterals$ is a {\atypeofmodel$_1$}-model of
$\genleftgen(\aprogram)$ such that
$\genrighttest(\aprogram,\astringofliterals)$ has no {\atypeofmodel$_2$}-model.

\medskip
{\em Statement 2.}
Assume $\failstate$ is not reachable from the
initial state. Then, since the graph is acyclic, there is a terminal state
different from $\failstate$. By Claim 1, this state is
$\terminalstate{\astringofliterals}$, and $\astringofliterals$ is a {\atypeofmodel$_1$}-model of
$\genleftgen(\aprogram)$ such that
$\genrighttest(\aprogram,\astringofliterals)$ has no {\atypeofmodel$_2$}-model.

Assume $\failstate$ is reachable from the
initial state. As it is different from the initial state there is a
transition leading to it. This transition can only be $\failleft$ or
$\failcross$. Let us call $(\astringofliterals,
 \anotherstringofliterals)_\aside$ a state from which a transition
leads to $\failstate$.
In either of these cases,
$\astringofliterals$ does not contain any decision literal;
so by Lemma~\ref{lemma:general},
any {\atypeofmodel$_1$}-model $\asetofliterals$ of
$\genleftgen(\aprogram)$ such that
$\genrighttest(\aprogram,\asetofliterals)$ has no {\atypeofmodel$_2$}-model
satisfies all the literals of $\astringofliterals$.
In other words, $\astringofliterals$ is the only possible candidate for
a {\atypeofmodel$_1$}-model of
$\genleftgen(\aprogram)$ such that
$\genrighttest(\aprogram,\astringofliterals)$ has no {\atypeofmodel$_2$}-model.

\medskip\emph{Case $\failleft$.} It follows that
$\astringofliterals$ is inconsistent. Consequently, it is not
a {\atypeofmodel$_1$}-model $\asetofliterals$ of
$\genleftgen(\aprogram)$ such that
$\genrighttest(\aprogram,\astringofliterals)$ has no {\atypeofmodel$_2$}-model.

\medskip\emph{Case $\failcross$.}
By Proposition~\ref{lemma:statement-e} $(d)$, the set of literals
$\anotherstringofliterals$ is a
{\atypeofmodel$_2$}-model of $\genrighttest(\aprogram,\astringofliterals)$. Thus
 $\astringofliterals$ is not a {\atypeofmodel$_1$}-model of $\genleftgen(\aprogram)$
such that $\genrighttest(\aprogram,\astringofliterals)$ has no {\atypeofmodel$_2$}-model.
\end{proof}

\medskip
\begin{proof}[Proof of Proposition~\ref{prop:approx-ensuring}]
We first illustrate that any set $\asetofliterals$ of literals  that  is a
{\atypeofmodel$_1$}-model of $\genleftgen(\aprogram)$ such that
$\genrighttest(\aprogram,\asetofliterals)$ has no {\atypeofmodel$_2$}-model is such that
$\restriction{\asetofliterals}{\atoms{\aprogram}}$ is a
stable model of $\aprogram$.
Indeed,
 by the definition of
{\atypeofmodel$_1$}-approximating functions w.r.t. {\atypeofmodel},
$\restriction{\asetofliterals}{\atoms{\aprogram}}$
is a {\atypeofmodel}-model of~$\aprogram$.
Also, by the definition of
 {\atypeofmodel$_2$}-ensuring functions w.r.t. {\atypeofmodel},
$\restriction{\asetofliterals}{\atoms{\aprogram}}$ is a
stable model of $\aprogram$.

Second, consider any stable model $\astringofliterals$ of $\aprogram$. By the definitions
of {\atypeofmodel$_1$}-approximating and {\atypeofmodel$_2$}-ensuring functions
w.r.t. {\atypeofmodel}, it follows there is
$\asetofliterals'$ such that
$\restriction{\asetofliterals'}{\atoms{\aprogram}}=\astringofliterals$ and
$\asetofliterals'$ is a {\atypeofmodel$_1$}-model of $\genleftgen(\aprogram)$
such that $\genrighttest(\aprogram,\asetofliterals')$ has no {\atypeofmodel$_2$}-model.
\end{proof}

\medskip
\begin{proof*}[Proof of Theorem~\ref{STS-correctness}]
Let {\atypeofmodel$_1$} denote a type
such that $\asetofpropagators_\genleftgen$ is
{\atypeofmodel$_1$}-enforcing and the function $\genleftgen$ is
{\atypeofmodel$_1$}-approximating w.r.t. {\atypeofmodel}.
Let {\atypeofmodel$_2$} denote a type
such that $\asetofpropagators_\genrighttest$ is
{\atypeofmodel$_2$}-enforcing and function $\genrighttest$ is
{\atypeofmodel$_2$}-ensuring w.r.t. {\atypeofmodel}.
We now proceed to prove the four conditions of the definition of `checks' one by
one.
\setlength\leftmargin{25pt}
\begin{enumerate}\itemsep0pt\parskip0pt\parsep0pt
  \item By Lemma~\ref{lemma:finite}, the graph
        $\simpletemplate{\asetofpropagators_\genleftgen}{\asetofpropagators_\genrighttest}
                        {\genleftgen}{\genrighttest}(\aprogram)$
        is acyclic and finite.

  \item By Lemma~\ref{lemma:structure} item 1, any terminal state is either
        $\failstate$ or $\terminalstate{\astringofliterals}$.

  \item By Lemma~\ref{lemma:structure} item 1, any terminal state of
        $\simpletemplate{\asetofpropagators_\genleftgen}{\asetofpropagators_\genrighttest}
                        {\genleftgen}{\genrighttest}(\aprogram)$
        reachable from the initial state and other than $\failstate$ is
        $\terminalstate{\astringofliterals}$, with
        $\seenasamodel{\astringofliterals}$ being a {\atypeofmodel$_1$}-model of
        $\genleftgen(\aprogram)$ such that
        $\genrighttest(\aprogram,\astringofliterals)$ has no
        {\atypeofmodel$_2$}-model.
        By Proposition~\ref{prop:approx-ensuring},
        $\restriction{\astringofliterals}{\atoms{\aprogram}}$ is a stable model
        of $\aprogram$.

  \item By Lemma~\ref{lemma:structure} item 2, $\failstate$ is reachable from
        the initial state iff $\genleftgen(\aprogram)$ has no
        {\atypeofmodel$_1$}-model $\astringofliterals$ such that
        $\genrighttest(\aprogram,\astringofliterals)$ has no
        {\atypeofmodel$_2$}-model.
        By Proposition~\ref{prop:approx-ensuring}, $\aprogram$ has no stable
        models.\mathproofbox
\end{enumerate}

\end{proof*}

\medskip
\begin{proof*}[Proof of Proposition~\ref{thm:firstgraph}]
Recall how we argued $\dpgraph{\aprogram}$ is
$\simpletemplate{\uppropagators}{\uppropagators}{\thegencmodels}{\thetestcmodels}(\aprogram)$.
Similarly, $\firstgraph{\themainprogram}{\thedependingprogram}(\aprogram)$ is
$\simpletemplate{\uppropagators}{\uppropagators}{\themainprogram}{\thedependingprogram}(\aprogram)$.
By Theorem~\ref{thm:typecomplete}, $\uppropagators$ is {\classic}-enforcing.

\setlength\leftmargin{25pt}
\begin{enumerate}\itemsep0pt\parskip0pt\parsep0pt
  \item By Lemma~\ref{lemma:finite},
        $\firstgraph{\themainprogram}{\thedependingprogram}(\aprogram)$ is
        finite and acyclic.

  \item By Lemma~\ref{lemma:structure} item 1, any terminal state is either
        $\failstate$ or $\terminalstate{\astringofliterals}$.

  \item By Lemma~\ref{lemma:structure} item 1, any terminal state of
        $\simpletemplate{\asetofpropagators_\genleftgen}{\asetofpropagators_\genrighttest}
        {\genleftgen}{\genrighttest}(\aprogram)$
        reachable from the initial state and other than $\failstate$ is
        $\terminalstate{\astringofliterals}$, with
        $\seenasamodel{\astringofliterals}$ being a {\classic}-model of
        $\genleftgen(\aprogram)$ such that
        $\genrighttest(\aprogram,\astringofliterals)$ has no {\classic}-model.

  \item By Lemma~\ref{lemma:structure} item 2, $\failstate$ is reachable from
        the initial state iff $\genleftgen(\aprogram)$ has no
        {\classic}-model $\astringofliterals$ such that
        $\genrighttest(\aprogram,\astringofliterals)$ has no
        {\classic}-model.\mathproofbox
\end{enumerate}
\end{proof*}

\subsection{Proof of Theorem~\ref{BigComp}}
First we prove an auxiliary lemma that will help handling CNF
conversions of DNF formulas.

For a DNF formula $\aformula$, we define $CNF(\aformula)$ as the conversion of
$\aformula$ to CNF using straightforward equivalent transformations: the
distributivity of disjunction over conjunction.
\begin{lemma}~\label{lemma:dnfcnf}
Let $\aformula$ be a DNF formula. Let $\aliteral$ be a literal of $\aformula$.
Let $\asetofliterals$ be a set of literals.

The two following statements are equivalent:
\setlength\leftmargin{25pt}
\begin{enumerate}\itemsep0pt\parskip0pt\parsep0pt
  \item there is a conjunctive clause $\aconjunction$ of $\aformula$ such that
        for every conjunctive
        clause $\aconjunction'\in \aformula$ different from $\aconjunction$,
        $\aconjunction'$ is contradicted by $\asetofliterals$,
  \item there is a clause $\aclause$ of $CNF(\aformula)$ such that
        $\aliteral\in\aclause$ and $\asetofliterals$
        contradicts $\aclause\setminus\{\aliteral\}$.
\end{enumerate}
\end{lemma}
\begin{proof}
Formula $\aformula$ has the form
$\bigvee_{\anindex=1}^\afourthindex\bigwedge_{\anotherindex=1}^\athirdindex
\aliteral_{\anindex \anotherindex}$ (when necessary the true constant
$\logicaltrue$ is added multiple times to ensure that the conjunctive clauses of $\aformula$ are of equal length).
Also $CNF(\aformula)=\bigwedge_{(\athirdindex_1\ldots
\athirdindex_\afourthindex)\in\{1\ldots
\athirdindex\}^\afourthindex}
\bigvee_{\anindex=1}^\afourthindex \aliteral_{\anindex \athirdindex_\anindex}$.

{\em From Statement 1 to Statement 2}: Assume that there is a conjunctive clause $\aconjunction$
of $\aformula$ such that for any other conjunctive clause $\aconjunction'$ of
$\aformula$, this clause is contradicted by $\asetofliterals$.
Let $\aliteral$ be a literal of $\aconjunction$. Let $\aconjunction$
be $\bigwedge_{\anotherindex=1}^\athirdindex \aliteral_{\anindex_0
\anotherindex}$ for some $\anindex_0$.
As any other conjunctive clause is contradicted by $\asetofliterals$, and as these clauses
are conjunctions, there is least one literal of each of these clauses that is
contradicted by $\asetofliterals$. Let us call
$\anotherliteral_1 \ldots \anotherliteral_{\anindex_0 - 1}
 \anotherliteral_{\anindex_0 + 1} \ldots \anotherliteral_\afourthindex$ these
literals.
Then for each $\anindex\in\{1,\ldots,\anindex_0-1,\anindex_0+1,\ldots,\afourthindex\}$, there
is $\athirdindex_\anindex^0\in\{1,\ldots \athirdindex\}$ such that
$\aliteral_{\anindex,\athirdindex_\anindex^0}=\anotherliteral_\anindex$.
Also, there is some $\athirdindex_{\anindex_0}^0$ such that
$\aliteral_{\anindex,\athirdindex_{\anindex_0}^0}=\aliteral$.
Then the clause $\bigvee_{\anindex=1}^\afourthindex \aliteral_{\anindex
\athirdindex_\anindex^0}$ of $CNF(\aformula)$ contains $\aliteral$ while each
of the other literals it contains is contradicted by $\asetofliterals$.

{\em From Statement 2 to Statement 1}: Assume that for some clause of $CNF(\aformula)$, 
all literals but one are known to be
contradicted by $\asetofliterals$. Then let this clause be
$\bigvee_{\anindex=1}^\afourthindex \aliteral_{\anindex \athirdindex_\anindex}$
for some $\anindex$ and let $\aliteral_{\anindex_0 \athirdindex_{\anindex_0}}$ be the
literal that is not contradicted by $\asetofliterals$. Then
$\aliteral_{\anindex \athirdindex_\anindex}$ is contradicted by
$\asetofliterals$ for any $\anindex$ other than $\anindex_0$.
So $\bigwedge_{\anotherindex=1}^\athirdindex \aliteral_{\anindex \anotherindex}$
is contradicted by $\asetofliterals$ for any $\anindex$ other than $\anindex_0$.
So $\aconjunction=\bigwedge_{\anotherindex=1}^\athirdindex
 \aliteral_{\anindex_0 \anotherindex}$ is a conjunctive clause of $\aformula$
such that for any other conjunctive clause $\aconjunction'$ of $\aformula$,
this clause is contradicted by $\asetofliterals$.
\end{proof}

\medskip
\begin{proof}[Proof of Theorem~\ref{BigComp}]
We must prove that for any edge in the graph
$\simpletemplate{\smdisjpropagators}{\asetofpropagators_\genrighttest}
                {\identity}{\genrighttest}(\aprogram)$
there is an edge in
$\simpletemplate{\uppropagators}{\asetofpropagators_\genrighttest}
                {\thecnfcomp}{\genrighttest}(\aprogram)$
linking two identical vertexes,
and for any edge in the graph
$\simpletemplate{\uppropagators}{\asetofpropagators_\genrighttest}
                {\thecnfcomp}{\genrighttest}(\aprogram)$
there is an edge in
$\simpletemplate{\smdisjpropagators}{\asetofpropagators_\genrighttest}
                {\identity}{\genrighttest}(\aprogram)$
linking two identical vertexes.

If the edge is justified by a right-rule then this is obvious as these two
graphs have the same witness function and the same set of conditions for the
$\propagateright$ rule.
If the edge is $\decideleft$, $\failleft$, $\backtrackleft$,
$\backtrackcross$ or $\failcross$ then
obviously there is the same edge in the other graph, bearing the same
name, as these edges do not depend on the generating program or set of
conditions for the $\propagateright$ rule.

It remains to study the case of an edge justified by $\propagateleft$ or
$\crossrule$.
Assume we also have proved that $\propagateleft$ rules are identical in both
graphs. Then if an edge is justified by $\crossrule$ in one of the graphs, which
means that no left-rule applies in this graph, equivalently no left-rule
applies in the other graph, and $\crossrule$ also applies in that graph.
We now show that $\propagateleft$ rules are identical in both
graphs, which will complete the proof.

Assume that an edge is justified by $\propagateleft$ in one of the
graphs, let us prove it also exists in the other graph.

\paragraph{A transition in
$\simpletemplate{\smdisjpropagators}{\asetofpropagators_\genrighttest}
                {\identity}{\genrighttest}(\aprogram)$
is justified by $\propagateleft$ with $\theunit$ as condition.}
Then also there
is an edge in
$\simpletemplate{\uppropagators}{\asetofpropagators_\genrighttest}
                {\thecnfcomp}{\genrighttest}(\aprogram)$
with the same effect, and justified by $\propagateleft$ with the $\theunit$ condition.
Indeed, $\aprogram$ is part of 
$\thecnfcomp(\aprogram)$.

\paragraph{A transition in
$\simpletemplate{\smdisjpropagators}{\asetofpropagators_\genrighttest}
                {\identity}{\genrighttest}(\aprogram)$
is justified by $\propagateleft$ with $\theallcancel$ as condition.}
Then the edge is turning
$(\astringofliterals,\emptyset)_\leftstate$
into $(\astringofliterals\logicalnot\anatom,\emptyset)_\leftstate$, and
each rule
$\ahead\logicalor\anatom \aspimplication \abody \in \aprogram$
is not a supporting rule for $\anatom$ w.r.t.~$\astringofliterals$.
In other words, for each rule
$\ahead\logicalor\anatom \aspimplication \abody \in \aprogram$
the following holds $\astringofliterals\cap(\opposite{\abody}\cup{\ahead})\neq\emptyset$.
Consequently, the conjunction $\abody\wedge\opposite{\ahead}$ is contradicted by $\astringofliterals$.
As a consequence 
$\bigvee_{\ahead\logicalor\anatom \aspimplication\abody\in\aprogram}
(\abody\logicaland\opposite{\ahead})$
is contradicted by $\astringofliterals$. From Lemma \ref{lemma:dnfcnf}, the fact that 
the DNF formula
$\logicalnot\anatom\logicalor
\bigvee_{\ahead\logicalor\anatom\aspimplication\abody\in\aprogram}
(\abody\logicaland\opposite{\ahead})$
belongs to $\completion{\aprogram}$, and the $\thecnfcomp$ construction, it follows that there is a
clause $\aclause$ in $\thecnfcomp(\aprogram)$
such that $\logicalnot\anatom\in\aclause$ and
$\astringofliterals$ contradicts $\aclause\setminus\{\logicalnot\anatom\}$.
So the rule $\propagateleft$ with condition $\theunit$ of
$\simpletemplate{\uppropagators}{\asetofpropagators_\genrighttest}
                {\thecnfcomp}{\genrighttest}(\aprogram)$
can be applied to $\aclause$ to add $\logicalnot\anatom$,
providing the edge we needed.

\paragraph{A transition in
$\simpletemplate{\smdisjpropagators}{\asetofpropagators_\genrighttest}
                {\identity}{\genrighttest}(\aprogram)$
is justified by $\propagateleft$ with $\thebacktrue$ as condition.} 
The proof of this case is similar to the proof of previous case.

\paragraph{A transition in
$\simpletemplate{\uppropagators}{\asetofpropagators_\genrighttest}
                {\thecnfcomp}{\genrighttest}(\aprogram)$
is justified by $\propagateleft$ with the condition $\theunit$.}
Let us call $\aformula_0$ the DNF formula
$\logicalnot\anatom\logicalor
\bigvee_{\ahead\logicalor\anatom\aspimplication\abody\in \aprogram}
(\abody\logicaland\opposite{\ahead})$ of $\completion{\aprogram}$ for some atom $\anatom$ in $\aprogram$.

\medskip\emph{Case 1: $\unitleft$ is applied to a clause of
$\aprogram$ in $\thecnfcomp(\aprogram)$.}
Then $\propagateleft$ with the condition $\theunit$ itself provides the desired
edge in
$\simpletemplate{\smdisjpropagators}{\asetofpropagators_\genrighttest}
                {\identity}{\genrighttest}(\aprogram)$.

\medskip\emph{Case 2: $\unitleft$ is applied to a clause obtained from
$\aformula_0$ by the $\thecnfcomp$ conversion.}
Then by Lemma \ref{lemma:dnfcnf}, the $\thecnfcomp$ construction, and the
$\unitleft$ condition there is a conjunctive clause $\aconjunction$ of
$\aformula_0$ such that for every conjunctive clause $\aconjunction'$ in
$\aformula_0$ that is different from $\aconjunction$ the current
$\astringofliterals$ contradicts $\aconjunction'$.

\medskip\emph{Case 2.1: This conjunctive clause is $\logicalnot\anatom$.}
Then $\astringofliterals$ contradicts
$\bigvee_{\ahead\logicalor\anatom\aspimplication\abody\in\aprogram}
(\abody\logicaland\opposite{\ahead})$. It is easy to see that $\theallcancel$ provides the
desired edge.

\medskip\emph{Case 2.2: This conjunctive clause is
              some $\abody\logicaland\opposite{\ahead}$.}
Then $\astringofliterals$ contradicts $\logicalnot\anatom$ so $\anatom$ belongs
to $\astringofliterals$. Also $\astringofliterals$ contradicts all of
$\{\abody'\logicaland\opposite{\ahead'}|
\ahead'\logicalor\anatom\aspimplication\abody'
\in\aprogram\setminus\{\ahead\logicalor\anatom\aspimplication\abody\}\}$.
As a consequence $\thebacktrue$ provides the desired edge.
\end{proof}

\section{Conclusions, Future and Related Work}
\label{sec:conclusion}

Transition systems for describing \dpll-based solving procedures have been
introduced by \citeN{nie06}. 
\citeN{lier08} introduced and compared the transition
systems for the answer set solvers {\sc smodels} and {\cmodels} for
non-disjunctive programs. In this paper, we
continue this direction of work by presenting a two-layer framework suitable to capture disjunctive answer set solvers.
We argue that this framework
allows simpler analysis and comparison of these systems. 
We first introduce 
a general template that includes the techniques implemented in such solvers, and then define specific solvers by instantiating appropriate techniques using this template.
Formal results about the correctness of the abstract
representations are given.
We believe that this work is a stepping stone towards clear, comprehensive
articulation of main design features of current disjunctive answer set solvers
that will inspire new solving algorithms. Section~\ref{sec:applications} hints
at some of the possibilities. Indeed, 
 to obtain a new solver
one can combine any appropriately chosen approximating-pair and ensuring-pair.

\citeN{nie06} considered another
extension of the graphs by introducing transition rules that
capture backjumping and learning techniques common in design of modern
solvers, that later allowed \citeN{lier10} to design, e.g., abstract {\clasp}. 
It is a direction of future work to extend the two-layer template graph to model
such advances solving techniques. This extension will allow us to model
disjunctive answer set solvers that rely heavily on backjumping and learning such as {\clasp} and {\wasp}.

\paragraph{Related work.}
The
approach based on transition systems for describing and comparing ASP
procedures is one of the three main alternatives studied in the ASP literature.
Other methods include pseudo-code presentation of
algorithms~\cite{giumar05,giu08} and tableau calculi~\cite{gebsch06iclp,geb13}. 
\citeN{giu08} presented pseudo-code descriptions of
{\cmodels} without backjumping and learning,
{\smodels} and {\dlv} without backjumping restricted to non-disjunctive
programs. They study relationships to the solving algorithms by analyzing the
correspondence about the search spaces they explore, focusing on tight
programs: in particular, they note a tight relation between solvers {\cmodels}
and {\dlv}. 
\citeN{geb13} considered
formal proof systems based on tableau methods for characterizing the
operations and the strategies of ASP procedures for
disjunctive programs. These proof systems  also allow cardinality
constraints in the language of logic programs.

\bibliographystyle{acmtrans}


\begin{thebibliography}{}

\bibitem[\protect\citeauthoryear{Alviano, Dodaro, Faber, Leone, and
  Ricca}{Alviano et~al\mbox{.}}{2013}]{dod13}
{\sc Alviano, M.}, {\sc Dodaro, C.}, {\sc Faber, W.}, {\sc Leone, N.}, {\sc
  and} {\sc Ricca, F.} 2013.
\newblock {WASP:} {A} native {ASP} solver based on constraint learning.
\newblock In {\em Proceedings of the 12th International Conference of Logic
  Programming and Nonmonotonic Reasoning ({LPNMR} 2013)}, {P.~Cabalar} {and}
  {T.~C. Son}, Eds. Lecture Notes in Computer Science, vol. 8148. Springer,
  54--66.

\bibitem[\protect\citeauthoryear{Baral}{Baral}{2003}]{bar03}
{\sc Baral, C.} 2003.
\newblock {\em Knowledge Representation, Reasoning and Declarative Problem
  Solving}.
\newblock Cambridge University Press.

\bibitem[\protect\citeauthoryear{Brochenin, Lierler, and Maratea}{Brochenin
  et~al\mbox{.}}{2014}]{blm14}
{\sc Brochenin, R.}, {\sc Lierler, Y.}, {\sc and} {\sc Maratea, M.} 2014.
\newblock Abstract disjunctive answer set solvers.
\newblock In {\em Proceedings of the 21st European Conference on Artificial
  Intelligence (ECAI 2014)}. Frontiers in Artificial Intelligence and
  Applications, vol. 263. IOS Press, 165--170.

\bibitem[\protect\citeauthoryear{Brooks, Erdem, Erdo\u{g}an, Minett, and
  Ringe}{Brooks et~al\mbox{.}}{2007}]{bro07}
{\sc Brooks, D.~R.}, {\sc Erdem, E.}, {\sc Erdo\u{g}an, S.~T.}, {\sc Minett,
  J.~W.}, {\sc and} {\sc Ringe, D.} 2007.
\newblock Inferring phylogenetic trees using answer set programming.
\newblock {\em Journal of Automated Reasoning\/}~{\em 39}, 471--511.

\bibitem[\protect\citeauthoryear{Davis, Logemann, and Loveland}{Davis
  et~al\mbox{.}}{1962}]{dav62}
{\sc Davis, M.}, {\sc Logemann, G.}, {\sc and} {\sc Loveland, D.} 1962.
\newblock A machine program for theorem proving.
\newblock {\em Communications of the ACM\/}~{\em 5(7)}, 394--397.

\bibitem[\protect\citeauthoryear{Eiter and Gottlob}{Eiter and
  Gottlob}{1993}]{eit93a}
{\sc Eiter, T.} {\sc and} {\sc Gottlob, G.} 1993.
\newblock Complexity results for disjunctive logic programming and application
  to nonmonotonic logics.
\newblock In {\em Proceedings of the 1993 International Logic Programming
  Symposium (ILPS)}, {D.~Miller}, Ed. 266--278.

\bibitem[\protect\citeauthoryear{Eiter, Gottlob, and Mannila}{Eiter
  et~al\mbox{.}}{1997}]{eite-etal-97f}
{\sc Eiter, T.}, {\sc Gottlob, G.}, {\sc and} {\sc Mannila, H.} 1997.
\newblock {Disjunctive Datalog}.
\newblock {\em ACM Transactions on Database Systems\/}~{\em 22,\/}~3 (Sept.), 364--418.

\bibitem[\protect\citeauthoryear{Faber}{Faber}{2002}]{faber2002enhancing}
{\sc Faber, W.} 2002.
\newblock Enhancing efficiency and expressiveness in answer set programming
  systems.
\newblock Ph.D. thesis, Ph.D.\ dissertation, Vienna University of Technology.

\bibitem[\protect\citeauthoryear{Gebser, Kaufmann, and Schaub}{Gebser
  et~al\mbox{.}}{2013}]{geb13a}
{\sc Gebser, M.}, {\sc Kaufmann, B.}, {\sc and} {\sc Schaub, T.} 2013.
\newblock Advanced conflict-driven disjunctive answer set solving.
\newblock In {\em Proceedings of the 23rd International Joint Conference on
  Artificial Intelligence (IJCAI 2013)}, {F.~Rossi}, Ed. IJCAI/AAAI.

\bibitem[\protect\citeauthoryear{Gebser and Schaub}{Gebser and
  Schaub}{2006}]{gebsch06iclp}
{\sc Gebser, M.} {\sc and} {\sc Schaub, T.} 2006.
\newblock Tableau calculi for answer set programming.
\newblock In {\em Proceedings of the 22nd International Conference on Logic
  Programming (ICLP 2006)}, {S.~Etalle} {and} {M.~Truszczynski}, Eds. Lecture
  Notes in Computer Science, vol. 4079. Springer, 11--25.

\bibitem[\protect\citeauthoryear{Gebser and Schaub}{Gebser and
  Schaub}{2013}]{geb13}
{\sc Gebser, M.} {\sc and} {\sc Schaub, T.} 2013.
\newblock Tableau calculi for logic programs under answer set semantics.
\newblock {\em ACM Transaction on Computational Logic\/}~{\em 14,\/}~2, 15.

\bibitem[\protect\citeauthoryear{Gelfond and Lifschitz}{Gelfond and
  Lifschitz}{1988}]{gel88}
{\sc Gelfond, M.} {\sc and} {\sc Lifschitz, V.} 1988.
\newblock The stable model semantics for logic programming.
\newblock In {\em Proceedings of the 5th International Conference and Symposium
  on Logic Programming (ICLP/SLP 1988)}, {R.~Kowalski} {and} {K.~Bowen}, Eds.
  MIT Press, 1070--1080.

\bibitem[\protect\citeauthoryear{Gelfond and Lifschitz}{Gelfond and
  Lifschitz}{1991}]{gel91b}
{\sc Gelfond, M.} {\sc and} {\sc Lifschitz, V.} 1991.
\newblock Classical negation in logic programs and disjunctive databases.
\newblock {\em New Generation Computing\/}~{\em 9}, 365--385.

\bibitem[\protect\citeauthoryear{Giunchiglia, Leone, and Maratea}{Giunchiglia
  et~al\mbox{.}}{2008}]{giu08}
{\sc Giunchiglia, E.}, {\sc Leone, N.}, {\sc and} {\sc Maratea, M.} 2008.
\newblock On the relation among answer set solvers.
\newblock {\em Annals of Mathematics and Artificial Intelligence\/}~{\em
  53,\/}~1-4, 169--204.

\bibitem[\protect\citeauthoryear{Giunchiglia and Maratea}{Giunchiglia and
  Maratea}{2005}]{giumar05}
{\sc Giunchiglia, E.} {\sc and} {\sc Maratea, M.} 2005.
\newblock On the relation between answer set and {S}{A}{T} procedures (or,
  between smodels and cmodels).
\newblock In {\em Proceedings of the 21st International Conference on Logic
  Programming (ICLP 2005)}, {M.~Gabbrielli} {and} {G.~Gupta}, Eds. Lecture
  Notes in Computer Science, vol. 3668. Springer, 37--51.

\bibitem[\protect\citeauthoryear{Janhunen, Niemel\"{a}, Seipel, Simons, and
  You}{Janhunen et~al\mbox{.}}{2006}]{jan06}
{\sc Janhunen, T.}, {\sc Niemel\"{a}, I.}, {\sc Seipel, D.}, {\sc Simons, P.},
  {\sc and} {\sc You, J.-H.} 2006.
\newblock Unfolding partiality and disjunctions in stable model semantics.
\newblock {\em ACM Transactions on Computunational Logic\/}~{\em 7,\/}~1,
  1--37.

\bibitem[\protect\citeauthoryear{Koch, Leone, and Pfeifer}{Koch
  et~al\mbox{.}}{2003}]{Koch:1999:SMC:1624218.1624229}
{\sc Koch, C.}, {\sc Leone, N.}, {\sc and} {\sc Pfeifer, G.} 2003.
\newblock Enhancing disjunctive logic programming systems by sat checkers.
\newblock {\em Artificial Intelligence\/}~{\em 151,\/}~1-2, 177--212.

\bibitem[\protect\citeauthoryear{Leone, Faber, Pfeifer, Eiter, Gottlob, Perri,
  and Scarcello}{Leone et~al\mbox{.}}{2006}]{dlv03a}
{\sc Leone, N.}, {\sc Faber, W.}, {\sc Pfeifer, G.}, {\sc Eiter, T.}, {\sc
  Gottlob, G.}, {\sc Perri, S.}, {\sc and} {\sc Scarcello, F.} 2006.
\newblock The {DLV} system for knowledge representation and reasoning.
\newblock {\em ACM Transactions on Computational Logic\/}~{\em 7,\/}~3,
  499--562.

\bibitem[\protect\citeauthoryear{Leone, Rullo, and Scarcello}{Leone
  et~al\mbox{.}}{1997}]{leo97}
{\sc Leone, N.}, {\sc Rullo, P.}, {\sc and} {\sc Scarcello, F.} 1997.
\newblock Disjunctive stable models: Unfounded sets, fixpoint semantics, and
  computation.
\newblock {\em Information and Computation\/}~{\em 135(2)}, 69--112.

\bibitem[\protect\citeauthoryear{Lierler}{Lierler}{2005}]{lie05}
{\sc Lierler, Y.} 2005.
\newblock Cmodels: {S}{A}{T}-based disjunctive answer set solver.
\newblock In {\em Proceedings of the 8th International Conference on Logic
  Programming and Nonmonotonic Reasoning ({LPNMR} 2005)}, {C.~Baral},
  {G.~Greco}, {N.~Leone}, {and} {G.~Terracina}, Eds. Lecture Notes in Computer
  Science, vol. 3662. Springer, 447--452.

\bibitem[\protect\citeauthoryear{Lierler}{Lierler}{2008}]{lier08}
{\sc Lierler, Y.} 2008.
\newblock Abstract answer set solvers.
\newblock In {\em Proceedings of the 24th International Conference on Logic
  Programming (ICLP 2008)}, {M.~G. de~la Banda} {and} {E.~Pontelli}, Eds.
  Lecture Notes in Computer Science, vol. 5366. Springer, 377--391.

\bibitem[\protect\citeauthoryear{Lierler}{Lierler}{2010}]{lierphd}
{\sc Lierler, Y.} 2010.
\newblock SAT-based answer set programming.
\newblock Ph.D. thesis, University of Texas at Austin.

\bibitem[\protect\citeauthoryear{Lierler}{Lierler}{2011}]{lier10}
{\sc Lierler, Y.} 2011.
\newblock Abstract answer set solvers with backjumping and learning.
\newblock {\em Theory and Practice of Logic Programming\/}~{\em 11}, 135--169.

\bibitem[\protect\citeauthoryear{Lierler and Truszczynski}{Lierler and
  Truszczynski}{2011}]{lie11a}
{\sc Lierler, Y.} {\sc and} {\sc Truszczynski, M.} 2011.
\newblock Transition systems for model generators - a unifying approach.
\newblock {\em Theory and Practice of Logic Programming\/}~{\em 11,\/}~4-5,
  629--646.

\bibitem[\protect\citeauthoryear{Lifschitz}{Lifschitz}{1999}]{lif99a}
{\sc Lifschitz, V.} 1999.
\newblock {Answer Set Planning}.
\newblock In {\em {Proceedings of the 16th International Conference on Logic
  Programming (ICLP 1999)}}, {D.~D. Schreye}, Ed. The MIT Press, Las Cruces,
  New Mexico, USA, 23--37.

\bibitem[\protect\citeauthoryear{Marek and Truszczy\'nski}{Marek and
  Truszczy\'nski}{1999}]{mar99}
{\sc Marek, V.} {\sc and} {\sc Truszczy\'nski, M.} 1999.
\newblock Stable models and an alternative logic programming paradigm.
\newblock In {\em The Logic Programming Paradigm: a 25-Year Perspective}.
  Springer Verlag, 375--398.

\bibitem[\protect\citeauthoryear{Niemel{\"a}}{Niemel{\"a}}{1999}]{nie99}
{\sc Niemel{\"a}, I.} 1999.
\newblock Logic programs with stable model semantics as a constraint
  programming paradigm.
\newblock {\em Annals of Mathematics and Artificial Intelligence\/}~{\em 25},
  241--273.

\bibitem[\protect\citeauthoryear{Nieuwenhuis, Oliveras, and
  Tinelli}{Nieuwenhuis et~al\mbox{.}}{2006}]{nie06}
{\sc Nieuwenhuis, R.}, {\sc Oliveras, A.}, {\sc and} {\sc Tinelli, C.} 2006.
\newblock Solving {S}{A}{T} and {S}{A}{T} modulo theories: From an abstract
  {D}avis-{P}utnam-{L}ogemann-{L}oveland procedure to {D}{P}{L}{L}({T}).
\newblock {\em Journal of the ACM\/}~{\em 53(6)}, 937--977.

\bibitem[\protect\citeauthoryear{Perri, Scarcello, Catalano, and Leone}{Perri
  et~al\mbox{.}}{2007}]{PerriSCL07}
{\sc Perri, S.}, {\sc Scarcello, F.}, {\sc Catalano, G.}, {\sc and} {\sc Leone,
  N.} 2007.
\newblock Enhancing {DLV} instantiator by backjumping techniques.
\newblock {\em Annals of Mathematics and Artificial Intelligence\/}~{\em 51,\/}~2-4, 195--228.

\bibitem[\protect\citeauthoryear{Ricca, Grasso, Alviano, Manna, Lio, Iiritano,
  and Leone}{Ricca et~al\mbox{.}}{2012}]{ricca-etal-tplp-2012}
{\sc Ricca, F.}, {\sc Grasso, G.}, {\sc Alviano, M.}, {\sc Manna, M.}, {\sc
  Lio, V.}, {\sc Iiritano, S.}, {\sc and} {\sc Leone, N.} 2012.
\newblock Team-building with answer set programming in the gioia-tauro seaport.
\newblock {\em Theory and Practice of Logic Programming\/}~{\em 12,\/}~3,
  361--381.

\bibitem[\protect\citeauthoryear{Simons, Niemel{\"a}, and Soininen}{Simons
  et~al\mbox{.}}{2002}]{sim02}
{\sc Simons, P.}, {\sc Niemel{\"a}, I.}, {\sc and} {\sc Soininen, T.} 2002.
\newblock Extending and implementing the stable model semantics.
\newblock {\em Artificial Intelligence\/}~{\em 138}, 181--234.

\bibitem[\protect\citeauthoryear{Soininen and Niemel{\"a}}{Soininen and
  Niemel{\"a}}{1999}]{soin-niem-99}
{\sc Soininen, T.} {\sc and} {\sc Niemel{\"a}, I.} 1999.
\newblock {Developing a Declarative Rule Language for Applications in Product
  Configuration}.
\newblock In {\em Proceedings of the 1st International Workshop on Practical
  Aspects of Declarative Languages (PADL 1999)}, {G.~Gupta}, Ed. Lecture Notes
  in Computer Science, vol. 1551. Springer, 305--319.

\bibitem[\protect\citeauthoryear{Syrj{\"{a}}nen}{Syrj{\"{a}}nen}{2001}]{syrj-2001}
{\sc Syrj{\"{a}}nen, T.} 2001.
\newblock Omega-restricted logic programs.
\newblock In {\em Proceedings of the 6th International Conference on Logic
  Programming and Nonmonotonic Reasoning ({LPNMR} 2001)}, {T.~Eiter},
  {W.~Faber}, {and} {M.~Truszczynski}, Eds. Lecture Notes in Computer
  Science, vol. 2173. Springer, 267--279.
\end{thebibliography}


\end{document}